\documentclass[11pt]{article}
\usepackage{preamble}
\usepackage{times}

\title{Optimal inference schedules for masked diffusion models}
\author{
    Sitan Chen\thanks{Email: \texttt{sitan@seas.harvard.edu}. Work supported in part by NSF CAREER Award CCF-2441635.} \\
    Harvard
        \and 
    Kevin Cong \thanks{Email: \texttt{kcong@college.harvard.edu}} \\
    Harvard
        \and
    Jerry Li \thanks{Email: \texttt{jerryzli@cs.washington.edu}} \\
    UW}
\newcommand{\brk}[1]{[#1]}

\newcommand{\DTC}{\mathrm{DTC}}
\newcommand{\TC}{\mathrm{TC}}
\newcommand{\TV}{\mathsf{TV}}

\newcommand{\costTV}{\mathsf{cost}^{\sf TV}}
\newcommand{\costKL}{\mathsf{cost}^{\sf KL}}
\newcommand{\leaves}{\mathsf{leaf}}

\newcommand{\subspace}{\mathcal{V}}

\newcommand{\node}{N}

\renewcommand{\vec}{\mathbf}

\begin{document}
\maketitle
\thispagestyle{empty}
\pagenumbering{gobble}

\begin{abstract}
    A major bottleneck of standard auto-regressive large language models is that their inference process is inherently sequential, resulting in very long and costly inference times.
    To circumvent this, practitioners proposed a class of language models called \emph{diffusion language models}, of which the \emph{masked diffusion model} (MDM) is the most successful. The MDM is able to sample tokens out-of-order and, ostensibly, many tokens at once and in parallel.
    However, there is very limited rigorous understanding of how much parallel sampling these models can perform without noticeable degradation in their sampling performance.
    Prior work of Li and Cai~\cite{li2025convergence} obtained some preliminary bounds, but these are not tight for many natural classes of distributions.
    In this work, we give a new, \emph{exact} characterization of the expected divergence between the true distribution and the sampled distribution, for any distribution and any unmasking schedule for the sampler, showing an elegant connection to the theory of \emph{univariate function approximation}. 
    
    By leveraging this connection, we then attain a number of novel lower and upper bounds for this problem. 
    While the connection to function approximation in principle gives the optimal unmasking schedule for any distribution, we show that it is in general impossible to compete with it without strong \emph{a priori} knowledge of the distribution, even in seemingly benign settings. However, we also demonstrate new upper bounds and new sampling schedules in terms of well-studied information-theoretic properties of the base distribution, namely, its \emph{total correlation} and \emph{dual total correlation}, which show that in some natural settings, one can sample in $O(\log n)$ steps without any visible loss in performance, where $n$ is the total sequence length.
\end{abstract}
\newpage

\tableofcontents
\thispagestyle{empty}

\newpage
\pagenumbering{arabic}

\section{Introduction}\label{intro}

Diffusion models are the state-of-the-art approach to generative modeling over domains like video and molecules, and in recent years have also emerged as a powerful alternative~\cite{sahoo2024simple,shi2024simplified,nie2025large,khanna2025mercury} to autoregressive large language models (LLMs). Abstractly, these models perform distribution learning by learning to reverse a corruption process transforming data into noise. Starting from fresh noise, they apply the learned reverse process to map it into a fresh sample from the data distribution.

In state-of-the-art diffusion language models, the most common choice of corruption process is the \emph{erasure process}. This is the basis for the popular paradigm of \emph{masked diffusion models (MDMs)}, which now form the backbone of leading approaches to non-autoregressive language modeling. The erasure process proceeds as follows: starting from a sample $X^0 = x^0 \in \Sigma^n$ at time $t = 0$, draw times $T_1,\ldots,T_n$ from some measure on $[0,1]$, and define $X^t_i = x^0_i$ if $t \le T_i$ and $X^t_i = *$ otherwise, where ``$*$'' is a special character corresponding to erasure.
Given samples from data distribution $\mu$, one then trains a neural network to learn \emph{conditional marginals}: for every time $t$ and conditioning $X^t = x^t$, estimate $\mathrm{law}(X^0_i \mid X^t = x^t)$ for all $i \in [n]$. It is readily seen that this is equivalent to learning the conditional marginals 
\begin{equation}
    \mathrm{law}(X_i \mid X_S = z)\,, \qquad S\subseteq[n]\,, i\not\in S\,,
\end{equation}
where $X \sim \mu$, $z\in \Sigma^{|S|}$, and $X_S = z$ denotes the partial assignment to the indices given by $S$. Given these marginals, it is straightforward to sample from $\mu$ by sampling one token at a time. Unlike LLMs which sample one token at a time from left to right, MDMs can sample \emph{out of order}.

\vspace{0.3em}

\noindent \textbf{Sampling multiple tokens at a time.} Crucially in practice, the neural network that is trained to learn these conditional marginals can, given any such partial assignment $X_S = z$, simultaneously compute the conditional marginals for all $i\not\in S$ \emph{in one network evaluation}. One of the key selling points of MDMs is thus that these models have the freedom to sample \emph{multiple tokens} at a time in parallel, whereas LLMs are inherently limited to sequential sampling. Empirically, a standard heuristic is the following: fix an \emph{unmasking schedule} given by \emph{step sizes} $s_1, s_2, \ldots, s_k$ summing to $n$, and iterate the following for $t = 1,\ldots,k$:
\begin{itemize}[itemsep=0pt,topsep=0.5em]
    \item Sample a random subset $S_t$ of size $s_t$ from among the indices $[n]\backslash (S_1\cup\cdots\cup S_{t-1})$
    \item For every $i\in S_t$, sample $x_i$ \emph{independently} from $\mathrm{law}(X_i \mid X_{S_1\cup\cdots\cup S_{t-1}} = x_{S_1\cup\cdots\cup S_{t-1}})$\footnote{Of course, in reality we never have exact access to the conditional marginal, but in our theoretical analysis it is straightforward to decouple this error from the overall sampling error; see Appendix~\ref{app:decoupling}.} -- note that this ignores correlations across $S_t$ and thus introduces statistical error.
\end{itemize}
The goal is to make $k$ as small as possible while keeping the statistical error small.  If $k = n$ and $s_1 = \cdots = s_n = 1$, then this will perfectly sample from the data distribution $\mu$, but it is no more efficient than sampling with an autoregressive model. On the other hand, if $k = 1$ and $s_1 = n$, this will sample in one step but output the product distribution whose 1-wise marginals agree with $\mu$, which in general will not be a good approximation to $\mu$. 

In real-world deployments of MDMs, there is an art to picking the unmasking schedule to trade off between these extremes, giving rise to popular heuristics like the \emph{cosine schedule}~\cite{chang2022maskgit,shi2024simplified} and the \emph{log-linear schedule}~\cite{lou2024discrete,sahoo2024simple} in which $s_1,s_2,\ldots$ start out small and progressively increase. However, our understanding of how to pick these schedules, and how to rigorously quantify the statistical errors that arise from sampling multiple tokens in parallel, remains limited. In this work we therefore ask:
\begin{center}
    \emph{What is the optimal unmasking schedule for a given data distribution $\mu$ and target level of error?}
\end{center}

\noindent This is a challenge not just for theory, but for practice. Indeed, a large-scale ML benchmark~\cite{kang2025parallelbench} was released just weeks ago in an effort to systematically evaluate unmasking schedules for diffusion language models. But as we will see, this is a question that is particularly amenable to the lens of theory.

\subsection{Result 1: Optimal unmasking schedule}

Our first result is a tight and surprisingly simple theoretical characterization of the optimal unmasking schedule for any $\mu$. The result exposes an elegant connection to \emph{univariate function approximation}. To state the result, we first require some terminology.

\begin{definition}[Expected KL error]\label{def:expectedKL}
    Conditioned on a sequence of subsets $S_1,\ldots,S_k$ of sizes $s_1,\ldots,s_k$, let $\nu^{S_1,\ldots,S_K}$ denote the distribution over outputs $x$ generated by the sampling algorithm above. The notion of error we will consider in this work is the \emph{expected KL error}
    \begin{equation}
        \E[S_1,\ldots,S_k]{\KL{\mu}{\nu^{S_1,\ldots,S_k}}}\,,
    \end{equation}
    where the expectation is over subsets $S_i$ of size $s_i$ sampled according to the algorithm above.
\end{definition}

\begin{definition}[Left Riemann approximation]
    Given $\bZ = (Z_1,\ldots,Z_n)\in\R_{\ge 0}^n$ and nodes $1 = N_1 < \cdots < N_k < n$, define the \emph{left Riemann approximation} of $\bZ$ to be the \emph{$k$-step sequence $Z^{\vec{\node}}_1,\ldots, Z^{\vec{\node}}_n$} given by:
    \begin{equation}
        Z^{\vec{\node}}_j = \begin{cases}
            Z_{\node_a} & \text{if} \ \node_a \le j < \node_{a+1} \\
            Z_{\node_k} & \text{if} \ j \ge \node_k
        \end{cases}
    \end{equation}
    Given any sequences $\bZ = (Z_1,\ldots,Z_n)$ and $\bZ' = (Z'_1,\ldots,Z'_n)$, we can define the \emph{integration error} $\norm{\bZ - \bZ'}_{L^1} \triangleq \sum^n_{j=1} |Z_j - Z'_j|$. The $k$-step left Riemann approximation to $\bZ$ minimizing this integration error is:
    \begin{equation}
        \vec{\node}^{*,k} \triangleq \mathop{\mathrm{argmin}}_{1 = \node_1 < \cdots < \node_k < n} \norm{\bZ - \bZ^{\vec{\node}}}_{L^1}\,. \label{eq:bestkstep}
    \end{equation}
    Note that given $\bZ$, one can efficiently find the minimizing $\vec{\node}^{*,k}$ in polynomial time via dynamic programming.
\end{definition}

\noindent The central object in this work is the following sequence quantifying correlations within $\mu$:

\begin{definition}[Average mutual information curve]
    Given a random variable $X\sim \mu$ over $\Sigma^n$, define its \emph{information curve}, denoted $\bZ = \bZ(\mu)$ by
     \begin{equation}
        Z_j = Z_j (\mu) \triangleq \E[|S| = j-1, i\not\in S]{I(X_i; X_S)}\,, \qquad j\in[n]\,,
    \end{equation}
    i.e. the average mutual information between $X_i$ and $X_S$ for random $S\subseteq[n]$ of size $j-1$ and random $i\not\in S$. 
    
    By Han's inequality~\cite[Theorem 1.7]{wupolyanskiy}, we have that $0 = Z_1 \le Z_2 \le \cdots Z_n$.
\end{definition}

\noindent Our first main result is an exact characterization of the optimal expected KL error achievable by any $k$-step sampler, in terms of the \emph{piecewise approximability of the distribution's information curve}: 

\begin{theorem}[Optimal schedule given by best step approximation]\label{thm:main}
    Let $\mu$ be any distribution over $\Sigma^n$, and let $1 \le k \le n$. Let $\vec{\node}^{*,k}$ be the solution to Eq.~\eqref{eq:bestkstep} for $\mu$'s information curve $\bZ = \bZ(\mu)$.
    
    Then for any unmasking schedule $s_1,\ldots,s_k$, the expected KL error is given by
    \begin{equation}
        \E[S_1,\ldots,S_k]{\KL{\mu}{\nu^{S_1,\ldots,S_k}}} = \norm{\bZ - \bZ^{\vec{\node}}}_{L^1}\,,\qquad \text{for} \ \ \ \node_a \triangleq 1 + \sum^{a-1}_{t=1} s_t\ \ \  \forall \ a \in [k]\,.
    \end{equation}
    In particular, the schedule that minimizes the expected KL error is 
    \begin{equation}
        s_t = \node^{*,k}_{t+1} - \node^{*,k}_t\,, \qquad t\in [k]\,.
    \end{equation}
\end{theorem}

\noindent In  Figure~\ref{pictorial_rep}, we give a pictorial depiction of the expected error in Theorem~\ref{thm:main}.
\begin{figure}[h]
\centering
\begin{tikzpicture}[scale=0.85, >=stealth]

\draw[->] (1,0) -- (15,0) node[right] {$j$};
\draw[->] (1,0) -- (1,6.5);
\node[left] at (1,0) {$0$};
\node[above] at (1,6.4) {\large $Z_j$}; 

\def\Z{{0.0,0.9,1.9,2.8,3.4,3.8,4.3,4.7,5.2,5.7,6.1,6.4,6.7}}

\begin{scope}
  \foreach \j in {2,3,4,5}{
    \filldraw[fill=blue!8, draw=gray!30, line width=0.2pt]
      (\j-1,{\Z[0]}) rectangle (\j,{\Z[\j-2]});
  }
  \foreach \j in {7,8,9,10}{
    \filldraw[fill=blue!8, draw=gray!30, line width=0.2pt]
      (\j-1,{\Z[4]}) rectangle (\j,{\Z[\j-2]});
  }
  \filldraw[fill=blue!8, draw=gray!30, line width=0.2pt]
    (11,{\Z[9]}) rectangle (12,{\Z[10]});
  \filldraw[fill=blue!8, draw=gray!30, line width=0.2pt]
    (13,{\Z[11]}) rectangle (14,{\Z[12]});
\end{scope}

\foreach \i [evaluate=\i as \ip1 using \i+1] in {1,...,12}{
  \draw[blue, thick] (\i,{\Z[\i-1]}) -- (\ip1,{\Z[\ip1-1]});
}
\foreach \i in {1,...,13}{
  \filldraw[blue] (\i,{\Z[\i-1]}) circle (1.8pt);
  \draw[gray!25, dashed] (\i,0) -- (\i,{\Z[\i-1]});
}

\draw[ultra thick, red!80!black]
  (1,{\Z[0]}) -- (5,{\Z[0]})
  (5,{\Z[4]}) -- (10,{\Z[4]})
  (10,{\Z[9]}) -- (12,{\Z[9]})
  (12,{\Z[11]}) -- (14,{\Z[11]});

\foreach \p/\z in {5/{\Z[0]},10/{\Z[4]},12/{\Z[9]}}{
  \draw[red!80!black, line width=0.8pt, fill=white] (\p,{\z}) circle (3pt);
}

\foreach \x in {1,5,10,12,14}{
  \draw[gray, dashed] (\x,0) -- (\x,6.6);
}

\node[red!70!black, below left] at (3.5,3.0) {$s_1 = 4$}; 
\node[red!70!black, below] at (7.5,3.0) {$s_2 = 5$};
\node[red!70!black, below] at (11.0,3.0) {$s_3 = 2$};
\node[red!70!black, below] at (13.0,3.0) {$s_4 = 2$};

\node[blue, above left] at (5,{\Z[4]}) {$Z_5$};
\node[blue, above left] at (10,{\Z[9]}) {$Z_{10}$};
\node[blue, above left] at (12,{\Z[11]}) {$Z_{12}$};
\node[blue, above left] at (13,{\Z[12]}) {$Z_{13}$};

\node[blue, right] at (14,{\Z[12]}) {$Z$};
\node[red!80!black, right] at (14,{\Z[11]-0.08}) {$Z^N$};

\foreach \i in {1,...,12}{
  \draw[gray!60] (\i,0) -- (\i,-0.1);
  \node[below, scale=0.8] at (\i,-0.1) {\scriptsize $\i$};
}
\node[below, scale=0.8] at (13,-0.1) {\scriptsize $n = 13$};

\end{tikzpicture}
\caption{Discrete curve $\mathbf Z$ (blue) and left Riemann approximation $\mathbf Z^{\mathbf N}$ (red) for a sample $Z_i$ curve. The latter extends beyond the $Z_j$ curve to $n+1$ to show the final rectangle $Z_n - Z_{N_{k-1} + 1}$; note that this term is not present in a standard left Riemann approximation. Light blue background rectangles represent the Riemann approximation terms. The total area is $\|\bZ - \bZ^N\|_{L^1}$.}

\label{pictorial_rep}
\end{figure}
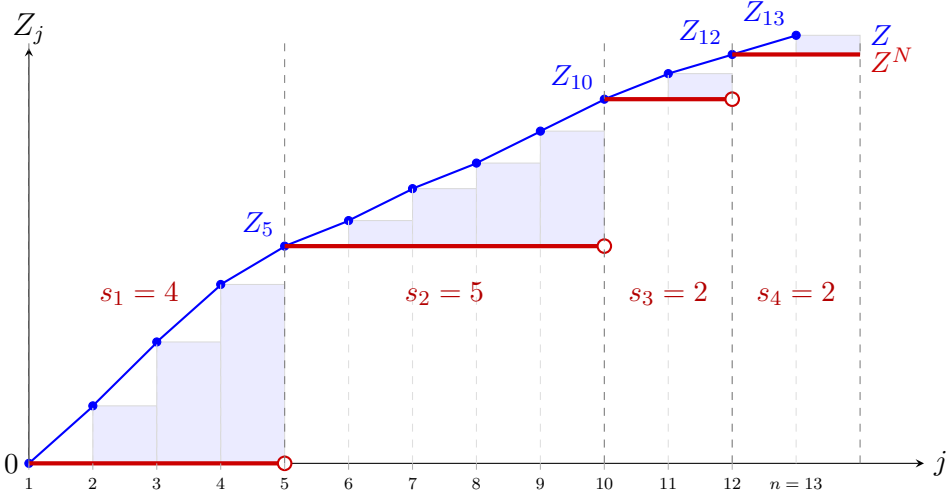

\noindent The proof of Theorem~\ref{thm:main} is remarkably simple once one realizes that the key object driving the statistical error of MDM sampling is the information curve of $\mu$, and we therefore regard the main technical contributions of this result as identifying the correct information-theoretic object to study as well as drawing the surprising connection to univariate function approximation.

\subsection{Result 2: Impossibility of competing with the optimal schedule}
\label{sec:impossibility}

Although Theorem~\ref{thm:main} gives an exact characterization of the optimal schedule, and this schedule can be found in polynomial time given \emph{a priori} knowledge of the information curve, pragmatically it is unclear how to use it as this \emph{a priori} knowledge is not readily available.\footnote{If one has $\mathrm{poly}(n,\epsilon)$ many held-out samples from $\mu$, one can always estimate each of the $Z_i$'s to sufficient precision, but in practice it can be prohibitively expensive to generate this many samples using the diffusion model.} One might hope that by making use of conditional marginal queries to the neural network, one can estimate $\bZ(\mu)$ to sufficient accuracy and then deduce the optimal schedule from this.

In the next part of this work, we prove a collection of impossibility results demonstrating that this is not possible in general, even under seemingly benign conditions. Our lower bounds in this part apply to an even more general setting where the sampling algorithm can adaptively make any $k$ conditional marginal queries it chooses (see Definition~\ref{def:oracle}), possibly in a randomized fashion, and then must output a sample such that marginally over its internal randomness, the algorithm's output distribution should be close to $\mu$.

We begin by considering a simple scenario where the curve is promised to either be the constant zero curve ($Z_j = 0$) or a single step function ($Z_j = \mathds{I}[j > j^*]$ for an unknown $j^*$), in which case the optimal schedule is simply determined by the location of the step, if it exists. There exist distributions realizing both kinds of information curve, namely the uniform distribution over $\Sigma^n$ and the uniform distribution over a minimum distance separable (MDS) code (see Definition~\ref{def:mds}). Our first result shows that even in this situation, finding $j^*$ if it exists requires a prohibitive number of conditional marginal queries.

\begin{theorem}[Uniform versus code is hard, see Theorem~\ref{thm:warmup} for formal statement]\label{thm:worstcase}
    There does not exist a single sampling algorithm which simultaneously achieves iteration complexity $o(n)$ for sampling to expected KL error $O(1)$ for all distributions, in fact even for sampling to expected TV error $1/2$. In fact, this holds even if the algorithm knew \emph{a priori} that the distribution were either the uniform distribution over $\Sigma^n$ or a uniform distribution over an unknown MDS code.
\end{theorem}

\noindent One might wonder if this worst-case result is too pessimistic, and in practice the relevant data distributions are very far from uniform and have useful correlational structure that one might hope to exploit. Unfortunately, the following strengthening of Theorem~\ref{thm:worstcase} shows that this is still not the case:

\begin{theorem}[Hardness for general information curves, see Theorem~\ref{thm:elevate} for formal statement]\label{thm:generalZ_informal}
    Let $\bZ = \bZ(\mu)$ be \emph{any} information curve, where $\mu$ is a distribution over $\mathbb{F}^n_q$, such that there exists an unmasking schedule under which one can sample from $\mu$ to expected KL error $O(1)$, in fact even just to expected TV error $1/2$. For every $1\le k < n$, let $\bZ^{\uparrow k}$ denote the information curve given by shifting every $Z_j$ for $j > k$ up by $\log_2(q)$.

    There does not exist a single sampling algorithm which simultaneously achieves iteration complexity $o(n)$, even if the algorithm knew \emph{a priori} that the information curve of the distribution was one of $\bZ, \bZ^{\uparrow 1}, \ldots, \bZ^{\uparrow n-1}$.
\end{theorem}

\noindent Intuitively, this result and the previous one follow from the fact that one can engineer sharp discrete jumps in the information curve which are not detectable unless if one conditions on exactly the right number of indices. We remark that the two results are technically incomparable as the hard distributions for Theorem~\ref{thm:generalZ_informal} when $\bZ$ is uniformly zero are slightly more intricate than simply uniform vs. MDS.

\subsection{Result 3: (Dual) total correlation and reduction to a single hyperparameter sweep}

In the final part of this paper, we redeem the situation by showing that for any distribution $\mu$, there exist unmasking schedules depending only on a \emph{single scalar parameter} quantifying correlations in the distribution which achieve small expected KL error.

For this, we need to first define two relevant information-theoretic quantities:

\begin{definition}[Total Correlation and Dual Total Correlation]\label{tc_dtc_defn}
    For any random variable $X \sim \mu$ over $\Sigma^n$, define the \emph{total correlation (TC)} as 
    \begin{equation}
        \TC = \TC(\mu) \coloneqq \Bigl(\sum_{i=1}^n H(X_i)\Bigr) - H(X_1, \ldots, X_n)
    \end{equation}
    and the \emph{dual total correlation (DTC)} as 
    \begin{equation}
        \DTC = \DTC(\mu):= H(X_1, \ldots, X_n) - \sum_{i=1}^n H(X_i \mid X_1, \ldots, X_{i-1}, X_{i+1}, \ldots, X_n)\,.
    \end{equation}
\end{definition}

\noindent From its definition, we see that $\TC$ is equivalently the KL divergence between $\mu$ and the product distribution whose marginals agree with those of $\mu$, and thus it characterizes how ``product'' $\mu$ is. On the other hand, $\DTC$ has been shown to quantify the extent to which $\mu$ can be expressed as a sparse \emph{mixture} of product distributions~\cite{austin2020multi}.  These quantities admit nice characterizations in terms of the information curve of $\mu$:
\begin{lemma}\label{tc_dtc_facts}
    For any distribution $\mu$ with information curve $\bZ$, \begin{enumerate}
        \item $\TC(\mu) = \sum_{i=1}^n Z_i$, and 
        \item $\DTC(\mu) = nZ_n - \sum_{i=1}^n Z_i = nZ_n - \TC(\mu)$. 
    \end{enumerate}
\end{lemma}

\noindent Under the pictorial representation in Figure~\ref{pictorial_rep}, $\TC$ is therefore the area \emph{under} the information curve, and $\DTC$ is the area \emph{above} the information curve (capped at $Z_n$).

We show that while it is not in general possible to compete with a sampler that can choose the unmasking schedule dependent on \emph{a priori} knowledge of the information curve, there are unmasking schedules that only depend on having access to a constant-factor approximations to $\TC(\mu)$ and $\DTC(\mu)$ which only require a number of iterations scaling in $\min(\TC(\mu),\DTC(\mu))$, up to log factors. In situations where these quantities are sublinear in $n$, this gives us a way to sample asymptotically faster than the naive $n$-step sampler even without full knowledge of the information curve. As a simple example, if $\mu$ is a distribution over a linear subspace of dimension or codimension $O(1)$ (e.g., if $\mu$ corresponds to an unknown parity), then this yields an \emph{exponential} speedup over naive schedules. We discuss other such examples in Section~\ref{sec:connections} below.

\begin{theorem}[Iteration complexity depending on $\TC, \DTC$]\label{thm:tcdtc}
    For any $\epsilon > 0$, there exists an unmasking schedule which depends only on $\epsilon$ and a parameter $\widehat{\TC}$ (resp. $\widehat{\DTC}$) such that for any distribution $\mu$ for which $\TC(\mu) \leq \widehat{\TC}$ (resp. $\DTC(\mu) \leq \widehat{\DTC}$), the expected KL error satisfies
    \begin{equation}
        \E[S_1,\ldots,S_k]{\KL{\mu}{\nu^{S_1,\ldots,S_k}}} \le \epsilon\,,
    \end{equation}
    and furthermore the number of steps satisfies
    \begin{equation}
        k \le 2 + (1 + \log n) \cdot (1 + \lceil \widehat{\TC} / \epsilon\rceil) \qquad (\text{resp.} \ k\le 2 + (1 + \log n) \cdot (1 + \lceil \widehat{\DTC} / \epsilon\rceil)\,.
    \end{equation} 
\end{theorem}

\noindent While realizing either (and in particular, the minimum) of these iteration complexities still requires knowing an upper-bound approximation of $\TC(\mu), \DTC(\mu)$, in practice this is not really an issue: one can simply treat these as hyperparameters and either estimate them with held-out data or guess their values via doubling. While the no-go results of Section~\ref{sec:impossibility} tell us it is impossible in theory to know when to stop doubling, in practice we can simply generate samples according to the different schedules and inspect at what point the output is sufficiently coherent. We emphasize that the reason this is feasible compared to the scheme suggested in Theorem~\ref{thm:main} is that we have reduced from designing a schedule that depends on $n$ different hyperparameters $Z_1,\ldots,Z_n$ (more than the number of hyperparameters describing the unmasking schedule itself) to designing one that only depends on $2$ hyperparameters, namely $\TC(\mu)$ and $\DTC(\mu)$.

Finally, the reader may wonder whether the $\log(n)$ factor in Theorem~\ref{thm:tcdtc} is a technical artifact or fundamental. In Appendix~\ref{sec:logn} we show that it is unavoidable, since there exist information curves which can only be approximated to $L^1$ error $\epsilon$ using step functions with at least $\Omega(\min(\TC, \DTC)\cdot \log(n)/\epsilon)$ steps.

\subsection{Related work}
\label{sec:connections}

We contrast our results with some existing bounds from the literature. 

\vspace{0.3em}\noindent\textbf{The bound of Li and Cai~\cite{li2025convergence}.} The most closely related prior work is that of Li and Cai~\cite{li2025convergence}. They considered the same setting as the present work and showed that under any unmasking schedule $s_1,\ldots,s_k$ with $s_{\max} \coloneqq \max_i s_i$, the expected KL error can be bounded by
\begin{equation}
    \frac{2^{\lceil \log_2 s_{\max}\rceil} - 1}{n} \sum^n_{i=1} I(X_i; X_1,\ldots,X_{i-1},X_{i+1},\ldots,X_n)\,. \label{eq:licai}
\end{equation}
This was proven using a delicate inductive argument based on recursively relating the expected KL error for a given unmasking schedule to the expected KL error with a schedule whose steps are twice as fine.

We make two observations about this bound. First, armed with Theorem~\ref{thm:main}, which gives an \emph{exact} characterization of the expected KL error for \emph{any} unmasking schedule, we can give a proof of Li and Cai's bound in just four lines --- see Appendix~\ref{sec:licai}. Second, we note that up to $\log n$ factors, the bound in Theorem~\ref{thm:tcdtc} is strictly better. The reason is that
\begin{equation}
    \sum^n_{i=1} I(X_i;X_1,\ldots,X_{i-1},X_{i+1},\ldots,X_n) = nZ_n = \TC(\mu) + \DTC(\mu) \asymp \max(\TC(\mu), \DTC(\mu))\,.
\end{equation}
For instance, in the aforementioned simple example where $\mu$ is distributed over a generic linear subspace, $\TC(\mu) + \DTC(\mu) = \Theta(n)$, whereas $\min(\TC(\mu),\DTC(\mu))$ scales with the minimum of the dimension and codimension, which can be much smaller (see Example~\ref{example:subspace}).

\vspace{0.3em}\noindent\textbf{DTC and the work of Tim Austin~\cite{austin2020multi,austin2019structure}.} The elegant work of Austin~\cite{austin2020multi} gave a powerful operational characterization of DTC. First, it is easily seen that any distribution $\mu$ which is a mixture of $2^{o(n)}$ product distributions has $\DTC(\mu) = o(n)$. Austin showed an approximate converse: if $\DTC(\mu) = o(n)$, then $\mu$ is well-approximated by a mixture of $2^{o(n)}$ product distributions. In fact, his proof is algorithmic and has an interesting interpretation under the perspective of the present work: it shows that if one first samples $O(\sqrt{\DTC(\mu) \cdot n/\epsilon})$ indices in sequence and then samples the remaining indices in $\calO\left(\sqrt{\DTC(\mu) \cdot n / \epsilon }\right)$ iterations, one can achieve expected KL error $\epsilon$:

\begin{theorem}[Austin's iteration complexity bound~\cite{austin2020multi}]\label{thm:austinquerycomplexity}
    For any $\epsilon > 0$, there exists an unmasking schedule which depends only on $\epsilon$ and a parameter $\widehat{\DTC}$ such that for any distribution $\mu$ for which $\DTC(\mu) \lesssim \widehat{\DTC}$, the expected KL error satisfies
    \begin{equation}
        \E[S_1,\ldots,S_k]{\KL{\mu}{\nu^{S_1,\ldots,S_k}}} \le \epsilon\,,
    \end{equation}
    and furthermore the number of steps satisfies $k \lesssim \left\lceil\sqrt{\DTC \cdot n/\epsilon}\right\rceil$.
\end{theorem}

\noindent We provide a proof for completeness in Appendix~\ref{sec:austin}. Although this bound is already sublinear in $n$ when $\DTC(\mu) = o(n)$, it is bottlenecked at $\sqrt{n}$. Indeed, note that Austin's bound is the geometric mean of our stronger bound in Theorem~\ref{thm:tcdtc} and the trivial iteration complexity bound of $n$, up to a logarithmic factor.

The connection between parallel sampling and decomposition of measure has been quite fruitful within probability theory. For instance, Austin~\cite{austin2019structure} showed the remarkably general result that for any Gibbs measure $\propto e^{-\beta H}$ of ``low-complexity'' in the sense that the discrete derivatives of the Hamiltonian lie within a set of bounded metric entropy, the $\DTC$ is $o(n)$. The decomposition of such distributions into mixtures of product measures arises naturally in the theory of nonlinear large deviations, see e.g., ~\cite{chatterjee2016nonlinear,eldan2018decomposition,eldan2018gaussian,eldan2018exponential}. The theory of parallel sampling can thus also be fruitfully interpreted as providing richer and more accurate hierarchical measure decompositions, where the levels of the hierarchy correspond to iterations of the sampler.

\vspace{0.3em}\noindent\textbf{Pinning lemma and stochastic localization.} Finally, we note a closely related notion from statistical physics and theoretical computer science, namely the \emph{pinning lemma}~\cite{raghavendra2012approximating,andrea2008estimating,el2022information}. The premise behind this result is that if one conditions on a random subset of coordinates of size $s$ according to their true $s$-wise marginal in $\mu$, then the remaining coordinates have pairwise correlation which is bounded by $O(1/s)$. In some sense this is the fundamental premise behind MDMs: pinning random tokens reduces the correlation among the remaining tokens, which intuitively enables more aggressive parallel sampling of later tokens. This has been used to great effect in the context of SDP rounding algorithms for solving dense CSPs~\cite{raghavendra2012approximating,manurangsi2017birthday,yoshida2014approximation,jain2019mean}

That being said, the pinning lemma holds for \emph{all} distributions, whereas our impossibility results show that without additional prior information about the distribution, one cannot simultaneously achieve $o(n)$ complexity for all $\mu$. It is worth contrasting this state of affairs with the work of~\cite{anari2024parallel}. By leveraging the pinning lemma, they showed that in a much \emph{stronger} parallel model where in each round one can simultaneously make multiple conditional marginal queries, each corresponding to a possibly \emph{different} partial assignment, it is possible to sample in $\tilde{O}(n^{2/3})$ parallel rounds, for general distributions.

\vspace{0.3em}\noindent\textbf{Other theoretical works on discrete diffusion.} We briefly mention some other works in the discrete diffusion literature that derive theoretical bounds. In~\cite{chen2024convergence}, the authors study discretization bounds for a different paradigm of discrete diffusions, where the corruption process being reversed is a bit flip channel rather than an erasure channel. Here, it is nontrivial even to derive bounds which scale linearly in $n$. In~\cite{ren2025fast}, the authors consider finding better discretizations of the continuous-time Markov chain associated to the discrete diffusion model; under some smoothness assumptions on the underlying distribution, which are primarily relevant to the bit flip setting, their higher-order solvers achieve nontrivial sampling guarantees relative to naive (Euler) discretization.

We also remark that the conditional marginal oracle we consider is very similar in spirit to the \emph{conditional query} model in distribution testing, pioneered by~\cite{chakraborty2013power,canonne2015testing}. Our model is different in two ways: (1) we restrict to \emph{subcube conditionings} in the sense of~\cite{canonne2021random}, and (2) a single oracle query gives an entire vector of 1-wise conditional marginals, rather than just a single sample from the posterior. The literature here is extensive and orthogonal to our work; we refer the interested reader to~\cite[Chapter 11]{canonne2020survey}. Lastly, we note that masked diffusion models -- and autoregressive models -- can be thought of as modern instantiations of the classical Jerrum-Valiant-Vazirani counting-to-sampling reduction~\cite{jerrum1986random}.

\vspace{0.3em}

\noindent \textbf{Continuous diffusion.} In recent years there has been significant progress on understanding discretization bounds for diffusion models over continuous spaces, e.g., the works of~\cite{Chen+23SGM,lee2023convergence,chen2023improved,Ben+24Diffusion, conforti2025kl,li2024sharp}. The techniques in this area are largely distinct from the ones in this work, with the exception of the recent work of~\cite{reeves2025information} which derived an analogous expression to our Theorem~\ref{thm:main} for the discretization error incurred by \emph{continuous} diffusions. In that context, the analogue of our information curve is the \emph{MMSE curve} $\E{X - \E{X \mid \alpha_t X + \beta_t \gamma}}$ for Gaussian $\gamma$, and they show (see Lemma 2 therein) that the discretization error is exactly given by the left Riemann integration error to this curve. Interestingly, whereas in the continuous diffusion context this integration error exactly characterizes the KL error in \emph{path space}, which is only an upper bound to the KL error at the endpoint of the sampler, in the masked diffusion setting the integration error exactly characterizes the sampler's KL error. In light of the connection to~\cite{reeves2025information}, it would be interesting to extend our impossibility results and TC/DTC-based bounds to the continuous setting.

\vspace{0.3em}\noindent \textbf{Concurrent work.} Independent concurrent work of~\cite{lavenant2025error} also identified the connection to Riemann approximation of the information curve (our ``Result 1''). Unlike our work, they did not explore the \emph{query complexity} of learning an optimal schedule (our ``Result 2'') and did not devise explicit schedules that scale better than the bound in~\cite{li2025convergence} (our ``Result 3''). Instead, they additionally provided worst-case bounds for sampling error under arbitrary, non-random orderings, and studied a natural $n\to \infty$ scaling limit of the step function approximation problem.

\section{Technical preliminaries}\label{prelims}

In this section, we will first provide a brief overview of our notation and oracle model. We will then discuss some important information theoretic quantities and results. 
\subsection{Notation}

Throughout the remainder of this paper, we will use the following notation. 


\vspace{0.3em}\noindent \textbf{Vocabulary, data distribution, and product Distributions.} We will let $\Sigma$ be a vocabulary and $\mu$ be the data distribution over $\Sigma^n$. We use $\Delta(\Sigma)$ to denote the probability simplex over $\Sigma$. Let $\mathbf X = (X_1, \ldots, X_{|S|}) \sim \mu$. For any set $S \subseteq [n]$, let $X_S = \{X_i\}_{i \in S}$. Define $\mu(\cdot \mid S)$ to be the conditional distribution of $\mu$ given $S$, and $\mu^\otimes (\cdot \mid S)$ to be the product distribution which has the same marginals as $\mu(\cdot \mid S)$. Lastly, define $f_\mu(\cdot \mid S)$ and $f_\mu^\otimes (\cdot \mid S)$ be the corresponding probability mass functions. Lastly, we set $X_S$ to denote the set $\{X_i\}_{i \in S}$.

\subsection{Oracle model}
Throughout this work, our main oracle object will be the \textit{conditional marginal oracle}, which outputs the marginals of $\mu$ conditioned on any subset. To define it, first recall that $(X_1, \ldots, X_n) \sim \mu$ is the data distribution over a vocabulary $\Sigma$. Let $\mathcal D$ be the collection of all multivariate distributions $\Sigma$. Moreover, let $\mathbf p_{i \mid S}(x_{S}) = \{p(X_i = j \mid X_S = x_S), j \in \Sigma\}$ be the marginal probability vector on coordinate $i$. We then have the following.

\begin{definition}[Conditional marginal oracle]\label{def:oracle}
   The \emph{conditional marginal oracle} $\mathsf{CO}$ takes as input a partial assignment $X_S = s$ and outputs the conditional marginal distributions of $\mu$ given $X_S = x_S$. Formally, $$\mathsf{CO}(X_i \mid X_S = x_S) = \{\mathbf \mu(X_i \mid X_S = x_S)\}_{i \not\in S}.$$ If the pinning $X_S = x_S$ is impossible in $\mathrm{supp}(\mu)$, output an arbitrary element of $\Delta(\Sigma)^{n-|S|}$.
\end{definition}

\noindent In our upper bounds, we will only ever use the oracle to obtain conditional marginals to sample from in parallel, as this is the standard way in practice to use this oracle. Our lower bounds however apply to the most general setting of arbitrary randomized algorithms with adaptive query access to $\mathsf{CO}$ (see Definition~\ref{def:sampling}).

Note that $\mathsf{CO}$ is an exact oracle, whereas in practice, an approximate oracle $\widehat{\mathsf{CO}}$ is learned from the training data. However, in Appendix \ref{app:decoupling}, we show following \cite{li2025convergence} that the error of our sampling algorithms can be decoupled into learning and sampling error. Since this work focuses on the sampling procedure and error, we will assume that the learned oracle is perfect. 
\subsection{Information-theoretic quantities}

In this section, we will recall some information theoretic quantities and prove a few preliminary lemmas which will be useful in the subsequent proofs of our main results. First, recall that $H(X)$ refers to the entropy of a random variable, and $H(Y \mid X)$ refers to the conditional entropy of a pair of random variables. Moreover, recall from Definition \ref{tc_dtc_defn} that $$\TC = \TC(\mu) := \left(\sum_{i=1}^n H(X_i)\right) - H(X_1, \ldots, X_n)$$ and $$\DTC = \DTC(\mu):= H(X_1, \ldots, X_n) - \sum_{i=1}^n H(X_i \mid X_1, \ldots, X_{i-1}, X_{i+1}, \ldots, X_n).$$
To provide some intuition for these quantities, we first provide some examples of the $\TC$ and $\DTC$ of linear subspace and product mixture distributions. 

\begin{example}[Linear Subspaces] \label{example:subspace}
Suppose $\Sigma = \mathbb F_q$ and $\subspace \subseteq \Sigma^n$  denote a linear subspace of dimension $d$. Let $\mu_{\subspace}$ denote the uniform distribution over $\subspace$. Then there is a matrix $M$ such that $MU\sim \mu_{\subspace}$ for a uniform $U \in \mathbb F_q^d$. Then  $$\TC(\mu) = (n-k)\log q - d \log q = (n-d - k) \log q$$ where $k = \#\{i: M_i = 0\}$ is the number of rows of $M$ which are identically 0. Moreover,  $$\DTC(\mu) = d\log q - \ell \log q = (d-\ell) \log q,$$ where $\ell = \#\left\{i: M_i \not\in \mathrm{span}\left(\{M_j\}_{j \neq i}\right)\right\}$ is the number of rows of $M$ which are not in the span of the remaining rows. In general, we will often have $k = \ell = 0$, in which case $\TC(\mu)$ and $\DTC(\mu)$ are the codimension and dimension of $\subspace$, respectively.
\end{example}

\begin{example}[Mixtures of Products]
    The DTC of product mixtures has been studied in-depth before in \cite{austin2020multi}. In particular, by Proposition 8.1 of \cite{austin2020multi}, the DTC of a mixture of $m$ product distributions is at most $\log m$. Thus, any mixture $\mu$ of $2^{o(n)}$ products satisfies $\DTC(\mu) = o(n)$. 
    
    In the converse direction, by Theorem A of \cite{austin2020multi}, any $\mu$ with $\DTC(\mu) = o(n)$ is close in transport distance to a relatively ``simple'' mixture of products.
\end{example}

\noindent Next, we define a sequence of values based on the average entropy of fixed-cardinality subsets. They will be useful to help analyze the information curve.

\begin{definition}[Average entropy curve]
    The \emph{average entropy curve} of the distribution $(X_1, \ldots, X_n) \sim \mu$ is given by\footnote{For $i = 0$, the entropy of the empty set is defined to be 0.} $$H_i(\mu) = \frac{1}{\binom ni} \sum_{S \subseteq [n], |S| = i} H\left(\{X_i\}_{i \in S}\right) = \mathbb E_{ |S| = i} H\left(\{X_j\}_{j \in S}\right)\,.$$
    When $\mu$ is clear from context, we denote this by $H_i$. 
\end{definition}

\noindent We can then express the information curve in terms of the average entropy curve as follows.
\begin{lemma}
    We have $Z_i = H_1 + H_{i-1} - H_i$. 
\end{lemma}
\begin{proof}
    Recall that $I(X_i; X_{S \backslash \{i\}}) = H(X_i) + H(X_{S\backslash \{i\}}) - H(X_S) $; taking expectations, we find that \begin{align*}Z_i &= \mathbb E_{|S| = i-1, j \not\in S}\left[I(X_j; X_{S})\right] \\&= \mathbb E_{|T| = i, j \in T} \left[H(X_j) + H(X_{T \backslash \{j\}}) - H(X_S)\right] \\ &= H_1 + H_{i-1 } - H_i,\end{align*} as desired. 
\end{proof}

\noindent We now provide several useful statements involving the information curve, namely that it provides a clean way to express the TC and DTC of a distribution. 

\begin{lemma}\label{tc_dtc_facts}
    We have the expressions \begin{enumerate}
        \item $\TC = \sum_{i=1}^n Z_i$ and 
        \item $\DTC = nZ_n - \sum_{i=1}^n Z_i = nZ_n - \TC$. 
    \end{enumerate}
\end{lemma}
\begin{proof}
First, observe that $$\sum_{i=1}^n Z_i = nH_1 - H_n = \sum_{i=1}^n H(X_i) - H(X_1, \ldots, X_n) = \TC,$$ which proves item 1. Lastly, using the chain rule of conditional entropy, we observe that \begin{align*}\DTC + \TC &= \sum_{i=1}^n (H(X_i) - H(X_i \mid X_1, \ldots, X_{i-1}, X_{i+1}, \ldots, X_n)) \\ &= \sum_{i=1}^n (H(X_i) + H(X_1, \ldots, X_{i-1}, X_{i+1}, \ldots, X_n) - H(X_1, \ldots, X_n)) \\ &= n(H_1 + H_{n-1} - H_n) \\ &= nZ_n.  \end{align*} Combining this equation with item 1, it follows that $$\DTC = nZ_n - TC= nZ_n - \sum_{i=1}^n Z_i,$$ which proves item 2.
\end{proof}

\section{Sampling error in terms of unmasking schedule: proof of Theorem \ref{thm:main}}\label{main}
In this section we establish an upper bound for the expected KL error of a fixed and then random unmasking algorithm, effectively proving Theorem \ref{thm:main}. We first formalize the definition of fixed unmasking algorithms as follows.

\begin{definition}[Fixed Unmasking Algorithm]
    The \textit{fixed unmasking algorithm} with subset schedule $(S_1, \ldots, S_k)$, given by $\calA_{\mathrm{fixed}}(k, \{S_i\}_{i = 1}^k)$, proceeds as follows. First define $N_i = \sum_{j=1}^i s_j$, where $N_0 = 0$ and $s_j = |S_j|$. Then at each stage $i \in [k]$, beginning at $i = 1$, independently and in parallel sample $$x_{j} \sim \mu\left(X_{j} \bigm| X_{\bigsqcup_{t = 1}^{i - 1} S_t} \right)$$ for all $j \in S_i$. The algorithm then outputs the sample $(x_1, \ldots, x_n) \sim \nu^{S_1, \ldots, S_k}$. 
\end{definition}

We next define the random unmasking algorithm as follows, based on the fixed unmasking algorithm. This is the formal definition for the algorithm that was outlined in Section~\ref{intro}.

\begin{definition}[Random Unmasking Algorithm]
    The \emph{random unmasking algorithm with unmasking schedule $(k, \{s_i\}_1^k)$}, given by $\calA(k, \{s_i\}_1^k)$, proceeds as follows. First, sample a uniformly random partition of coordinates $S = \sqcup_{i = 1}^k S_i, |S_i| = s_i$. Output a sample $(x_1, \ldots, x_n) \sim v^{S_1, \ldots, S_k}$ given by $\calA_{\mathrm{fixed}}(k, \{S_i\}_{i = 1}^k)$. The algorithm then outputs the sample $(x_1, \ldots, x_n) \sim \nu$. 
\end{definition}

\noindent Note that the fixed unmasking algorithm $\calA_{\mathrm{fixed}}$ essentially chooses $s_i$ tokens at fixed positions at each stage $i$ and samples all tokens independently and in parallel, and the random unmasking algorithm selects the token positions at each stage uniformly at random amongst all masked tokens. We can now formally state and prove (a slightly stronger form of) Theorem \ref{thm:main}.

\begin{theorem}\label{main_result}
    Let $\mu$ denote the underlying distribution of data. Let $(k, \{s_i\}_1^k)$ be an unmasking schedule and $\{S_i\}_1^k, |S_i| = s_i$ be a fixed subset schedule. Let $N_i = \sum_{j=1}^i s_j$ denote the partial sums of the $s_i$ sequence, where $N_0 = 0$. Suppose the fixed unmasking algorithm $\calA_{\mathrm{fixed}}(\{S_i\}_1^k)$ samples a distribution $\nu^{S_1, \ldots, S_k}$ and the random unmasking algorithm $\calA(k, \{s_i\}_1^k)$ samples a distribution $\nu$. Then both algorithms have query complexity $k$, and $\nu$ achieves KL error relative to $\mu$ of $$\KL \mu \nu \leq \mathbb E_{S_1, \ldots, S_k}\left[\KL \mu {\nu^{S_1, \ldots, S_k}} \right] = \sum_{i=1}^k \left(\sum_{j=1}^{s_i} \left(Z_{N_{i-1} + j} - Z_{N_{i-1} + 1}\right)\right),$$ where the expectation is taken over all partitions $S = \bigsqcup_{i=1}^k S_i$ for which  $|S_i|= s_i$. 
\end{theorem}
\begin{proof}
Let $T_i = \bigcup_{j < i} S_i$ with $T_1 = \emptyset$ be the coordinates which have already been sampled at stage $i$. We first work with the distribution $\nu^{S_1, \ldots, S_k}$. Observe that \begin{align*}\KL \mu {\nu^{S_1, \ldots, S_k}} &= \mathbb E_{\mathbf x \sim \mu} \left[\log \frac{f_\mu(\mathbf x)}{f_\mu^\otimes(\mathbf x)}\right] \\ &= \mathbb E_{\mathbf x \sim \mu} \left[\sum_{i=1}^k \log \frac{f_{\mu(X_{S_i} \mid X_{T_i} = x_{T_i})}}{f_{\mu^\otimes(X_{S_i} \mid X_{T_i} = x_{T_i})}}\right] \\ &= \sum_{i = 1}^k \mathbb E_{\mathbf x \sim \mu} \left[\KL{\mu(X_{S_i} \mid X_{T_i} = x_{T_i})}{\mu^\otimes(X_{S_i} \mid X_{T_i} = x_{T_i})}\right].\end{align*} To simplify this, observe that the inner KL term is essentially a total correlation of the conditional distribution of $X_{S_i}$ given $X_{T_i} = x_{T_i}$. Therefore, it follows that \begin{align*}&\mathbb E_{\mathbf x \sim \mu} \left[\KL{\mu(X_{S_i} \mid X_{T_i} = x_{T_i})}{\mu^\otimes(X_{S_i} \mid X_{T_i} = x_{T_i})}\right] \\ &\hspace{1.5cm}= \mathbb E_{\mathbf x \sim \mu} \left[\left(\sum_{j \in S_i} H(X_j \mid X_{T_i} = x_{T_i})\right) - H(X_{S_i} \mid X_{T_i} = x_{T_i})\right] \\ &\hspace{1.5cm}= \left(\sum_{j \in S_i} H(X_j \mid X_{T_i})\right) - H(X_{S_i} \mid X_{T_i}) \\ &\hspace{1.5cm}= \left(\sum_{j \in S_i} H(X_{T_i \cup \{j\}}) - H(X_{T_i})\right) - \left(H(X_{S_i \sqcup T_i}) - H(X_{T_i})\right) ,\end{align*} where in the second equality $H(X_j \mid X_{T_i} = x_{T_i})$ denotes the entropy of the conditional distribution of $X_j$ given $X_{T_i} = x_{T_i}$, while $H(X_j \mid X_{T_i})$ denotes the conditional entropy of $X_j$ given $X_{T_i}$.

Combining this with the previous equation, we find that \begin{align*}
    \KL \mu {\nu^{S_1, \ldots, S_k}} &= \sum_{i=1}^k \left[\left(\sum_{j \in S_i} H(X_{T_i \cup \{j\}}) - H(X_{T_i})\right) - \left(H(X_{S_i \sqcup T_i}) - H(X_{T_i})\right)\right].
\end{align*} Recall now that $\nu$ is given by the mixture $$\nu = \frac{1}{\binom{n}{s_1 \ldots s_k }} \sum_{\{S_i\}_1^k, |S_i| = s_i} \nu^{S_1, \ldots ,S_k}.$$ We therefore find that \allowdisplaybreaks
\begin{align*}
    \KL \mu \nu &= \KL \mu {\frac{1}{\binom{n}{s_1 \ldots s_k }} \sum_{\{S_i\}_1^k, |S_i| = s_i} \nu^{S_1, \ldots ,S_k}} \\ &\leq \frac{1}{\binom{n}{s_1 \ldots s_k }} \sum_{\{S_i\}_1^k, |S_i| = s_i} \KL \mu {\nu^{S_1, \ldots ,S_k}} \\ &= \mathbb E_{\{S_i\}_1^k, |S_i| = s_i}\left[\KL \mu {\nu^{S_1, \ldots, S_k}}\right] \\ &= \mathbb E_{\{S_i\}_1^k, |S_i| = s_i}\left[\sum_{i=1}^k \left[\left(\sum_{j \in S_i} H(X_{T_i \cup \{j\}}) - H(X_{T_i})\right) - \left(H(X_{S_i \sqcup T_i}) - H(X_{T_i})\right)\right]\right]  \\ &= \sum_{i=1}^k \Bigg[s_i\left(\mathbb E_{|S| = N_{i-1}+1}\left[H(X_S)\right] - \mathbb E_{|S| = N_{i-1}}\left[H(X_S)\right]\right) - \left(\mathbb E_{|S| = N_i}\left[H(X_S)\right] - \mathbb E_{|S| = N_{i-1}}\left[H(X_S)\right]\right)\Bigg] \\ &= \sum_{i=1}^k \Bigg[ s_i\left(H_{N_{i-1}+1} - H_{N_{i-1}}\right) - \left(H_{N_i} - H_{N_{i-1}}\right)\Bigg] \\ &= \sum_{i=1}^k \Bigg[s_i H_1 - s_i Z_{N_{i-1} + 1} - \sum_{j=1}^{s_i} (H_{N_{i-1} + j} - H_{N_{i-1} + j-1})\Bigg] \\ &= \sum_{i=1}^k \left[\left( \sum_{j=1}^{s_i} Z_{N_{i-1}+j}\right) - s_i Z_{N_{i-1}+1}\right],  
\end{align*} where the first line is an equality, the second by convexity of KL, the third and fourth lines are direct simplification, the fifth line follows from the fact that $T_i \cup \{j\}$, $T_i$, $S_i$, and $S_i \sqcup T_i$ are individually uniformly random subsets of $[n]$ of sizes $N_{i-1} + 1$, $N_{i-1}$, $s_i$, and $N_i$, respectively, and the final three lines are via directly applying the definitions of $H_i$ and $\bZ$. The theorem follows from the first, third, and final lines.
\end{proof}

\noindent We make two brief comments about Theorem \ref{main_result}.
\medskip

\noindent \textbf{Theorem \ref{main_result} and Theorem \ref{thm:main}.} First, we note that Theorem 1.4 is an immediate corollary. 

\begin{proof}[Proof of Theorem 1.4]
    Let $\mathbf N_a = 1 + \sum_{t = 1}^{a-1} s_t \hspace{0.2cm} \forall a \in [k]$ and $\mathbf N_0 = 1$. By Theorem \ref{main_result}, and the definition of $\mathbf Z^{\mathbf N}$, we have that $$\mathbb E_{S_1, \ldots, S_k} \left[\KL \mu {\nu^{S_1, \ldots, S^k}}\right] = \sum_{i = 1}^k \left(\sum_{j = 1}^{s_i} (Z_{N_{i-1} + j} - Z_{N_{i-1} + 1})\right) = \|\mathbf Z - \mathbf Z^{\mathbf N}\|_{L^1},$$ yielding the formula for KL error. The remainder of the theorem statement is obvious. 
\end{proof}

\noindent \textbf{Comparison between fixed and random unmasking algorithm.} There are two methods of approaching sampling: first, by fixing the schedule $S_i$ ahead of time, and second, by resampling $S_i$ from $|S_i| = s_i$ for each sample. These correspond to the fixed and random unmasking algorithms, respectively. We observe that the inequality in Theorem \ref{main_result} shows that the distribution outputted by the \textit{random} unmasking algorithm is, on average, superior in respect to KL-error from $\mu$ to the distribution outputted by the \textit{fixed} unmasking algorithm. This is an additional guarantee not given in Theorem \ref{thm:main}, and suggests that the random unmasking algorithm is superior, albeit requiring an additional step in the sampling process.

\section{Lower bounds on competing with the oracle rate}\label{lower_bounds}

\begin{definition}[MDS codes]\label{def:mds}
    A $k$-dimensional linear subspace $\subspace$ of $\mathbb{F}^n_q$ is an \emph{maximum distance separable (MDS) code} if for any $k\times n$ matrix $M$ whose rows constitute a basis for $\subspace$, every $k$ columns of $M$ are linearly independent. We denote by $\mathrm{Unif}(\subspace)$ the uniform distribution over points in $\subspace$.

    In this work, we will consider \emph{affine shifts} of MDS codes. That is, we will consider distributions over \emph{affine} subspaces which are given by taking some MDS code and translating it by a fixed vector in $\mathbb{F}^n_q$. We will abuse terminology and refer to such affine subspaces as MDS codes.
\end{definition}

\noindent We will consider ``random'' MDS codes:

\begin{definition}[Balanced random MDS codes]\label{def:balanced}
    A distribution $\calD$ over $k$-dimensional MDS codes is \emph{balanced} if for every subset $S\subseteq[n]$ for which $|S|\ge k$ and every partial assignment $x \in \mathbb{F}^{|S|}_q$,
    \begin{equation}
        \Pr[\subspace\sim \calD]{\exists \ x^*\in \subspace: x^*_S = x} = (1/q)^{|S| - k} 
    \end{equation}
\end{definition}

\noindent \emph{Reed-Solomon codes} provide an example of MDS codes. Below, we recall their definition.

\begin{definition}[Reed-Solomon codes]
     Let $q$ be any prime power exceeding $n$, and let $k$ be any value between $1$ and $n - 1$. A \emph{$k$-dimensional Reed-Solomon (RS) code in $\mathbb{F}^n_q$} is a linear subspace specified as follows. It is specified by a collection of distinct \emph{evaluation points} $a_1,\ldots,a_n\in\mathbb{F}_q$, and is given by the set of all evaluations $(p(a_1),\ldots,p(a_n))$ where $p$ is a polynomial over $\mathbb{F}_q$ of degree less than $k$.

     As in Definition~\ref{def:mds}, we will abuse terminology and also refer to affine shifts of RS codes as RS codes.
\end{definition}

\noindent We will leverage the following basic property of MDS codes:

\begin{proposition}\label{prop:independent}
    Let $\mu = \mathrm{Unif}(\subspace)$ for any $k$-dimensional MDS code $\subspace\subseteq\mathbb{F}^n_q$. Then for any $S\subseteq[n]$ satisfying $|S| < k$ and any partial assignment $x\in\mathbb{F}_q^{|S|}$, $\mu(X_i \mid X_S = x) = \mathrm{Unif}(\mathbb{F}_q)$ for all $i\not \in S$.

    In particular, this implies that $Z_j(\mu) = \log_2(q) \cdot \mathbb{I}[j> k]$.
\end{proposition}

\noindent In addition, recall from the definition of the oracle that if $|S| > k$ and the partial assignment $X_S = x$ is incompatible with any element of $\subspace$, then the output of the conditional marginal oracle can be arbitrary. Throughout this section, we will take the oracle's output in this case to be $\mathrm{Unif}(\mathbb{F}_q)^{\otimes(n-|S|)}$.

We will also use the following property of \emph{random} Reed-Solomon codes. Given prime power $q \ge n$ and dimension $0 < k < n$, let $\calD_{n,k,q}$ denote the following distribution over $k$-dimensional RS codes over alphabet $q$. When $n,q$ are clear from context, we denote this by $\calD_k$.

\begin{lemma}\label{lem:RS_balanced}
     $\calD_{k}$ is balanced in the sense of Definition~\ref{def:balanced}.    
\end{lemma}

Next, we formalize the model of computation under which we prove a lower bound.

\begin{definition}[Sampling algorithm]\label{def:sampling}
    Let $\calF\subseteq \Delta(\Sigma^n)$ denote some known family of distributions $\mu$. Given access to the conditional marginal oracle for some $\mu \in \calF$, an \emph{$\calF$-aware sampling algorithm $\calA$} is a procedure of the following form:
    \begin{enumerate}
        \item Repeat the following:
        \begin{itemize}
            \item Based only on the query outcomes from previous rounds and $\calF$ (and not on knowledge of $\mu$), and possibly using additional randomness, either query the oracle on a partial assignment $X_S = x$, or exit the loop.
            \item If the former, observe conditional marginals $\{\mu(X_i\mid X_S = x)\}_{i\not\in S}$
        \end{itemize}
        \item Output a string in $\Sigma^n$. 
    \end{enumerate}
    Importantly, the decision to exit out of the loop can be made adaptively. We say that $\calA$ is \emph{$T$-query} if with probability $1$ it performs at most $T$ queries to the oracle before terminating. We denote by $\calA[\mu]$ the distribution over outputs of $\calA$.
\end{definition}

\noindent Any such sampling algorithm can be naturally represented by a \emph{stochastic decision tree} as follows:

\begin{definition}[Stochastic decision tree representation]
    Any sampling algorithm $\calA$ can be regarded as an (infinite-degree) stochastic decision tree as follows. Every internal node is either a \emph{decision node} (including the root), a \emph{query node}, or a \emph{leaf node}. Decision and leaf nodes (resp. query nodes) are at even (resp. odd) distance from the root:
    \begin{itemize}[leftmargin=*]
        \item For every decision node $v$, the outgoing edges $(v,w)$ connect $v$ to query nodes $w$. Each such edge is labeled with a partial assignment $X_{S^{(w)}} = x^{(w)}$ with which to query the oracle. From $v$, the sampler transitions to $w$ with some probability $\Pr[\calA]{w\mid v}$.
        \item For every query node $w$, there is a continuum of infinitely many outgoing edges $(w,v')$, each labeled by an element of $\Delta(\Sigma)^{n - |S^{(w)}|}$ corresponding to a possible response by the oracle to the query $X_{S^{(w)}} = x^{(w)}$. From $w$, the sampler walks along the edge corresponding to the oracle's response to $X_{S^{(w)}} = x^{(w)}$.
        \item Each leaf node $\ell$ is labeled with a distribution $\nu_\ell$ over $\Sigma^n$, corresponding to the algorithm's (randomized) output if it has reached that state and decided to exit out of the loop. Let $\leaves(\mu)$ (resp. $\leaves^{\le T}(\mu)$) denote all possible leaf nodes of the stochastic decision tree corresponding to $\calA$ (resp. which are distance at most $2T$ from the root and reachable given oracle access to $\mu$). 
    \end{itemize}
    Every path from the root to a decision or leaf node $v$ is given by a path whose edges are alternatingly labeled by partial assignments $X_S = x$ and corresponding oracle responses.
    
    For any internal or leaf node $v$ of the tree, let $\Pr[\calA]{v\mid\mu}$ denote the probability that the algorithm traverses that node at some point in its execution, conditioned on the oracle responses coming from the conditional marginal oracle for $\mu$.
\end{definition}

\begin{definition}[Query budget and cost function]\label{def:budget}
    Let $\calT: \Delta(\Sigma^n)\to \mathbb{N}$ denote a \emph{query budget} for the query complexity of such a sampler, and define 
    \begin{equation}
        \costKL_\calT(\calA;\mu) \triangleq \begin{cases}
            \KL{\mu}{\calA[\mu]} & \calA \ \text{is at most} \ \calT(\mu)\text{-query} \\
            \infty & \text{otherwise}
        \end{cases}
    \end{equation}
    Define $\costTV_\calT(\calA;\mu)$ in the same way, with $\mathsf{KL}$ replaced by $\TV$.
\end{definition}

\subsection{Warmup example}
\label{sec:warmup}

We begin by exhibiting a simple ensemble of distributions for which no single algorithm can successfully sample from $\mu$ to error $\epsilon$ using $O(\min(\TC(\mu), \DTC(\mu)) \log(n)/\epsilon)$ for every $\mu$ in the ensemble.

Let $\calU$ denote the uniform distribution over $\mathbb{F}^n_q$, and let $\calF$ consist of $\calU$ as well as $\calU_\subspace$ for all Reed-Solomon codes $\subspace\subseteq\mathbb{F}^n_q$ of dimension $0 < k < n$. 

Formally, we show:

\begin{theorem}\label{thm:warmup}
    No $\calF$-aware sampling algorithm $\calA$ can achieve $\sup_{\mu\in\calF} \costTV_\calT(\calA;\mu) \le 1/16$ for any budget $\calT$ satisfying $\calT(\mu) \lesssim \max(1,\min(\TC(\mu),\DTC(\mu)))\log(n)$ for all $\mu \in \calF$.
\end{theorem}

\noindent We will use the following terminology: in the stochastic decision tree associated to $\calA$, a leaf $\ell$ is said to \emph{miss} a subspace $\subspace$ if, for all partial assignments labeling edges from the root-to-leaf path to $\ell$, either the assignment is of size less than $\dim \subspace$, or if otherwise there does not exist $x^* \in \subspace$ consistent with that assignment. Otherwise, $\ell$ is said to \emph{hit} $\subspace$.

\begin{proof}
    Let $\calD$ denote the mixture distribution over $\calF$ given by
    \begin{equation}
        \frac{1}{2}\delta_{\calU} + \frac{1}{2n-2}\sum^{n-1}_{k=1} \calD_k
    \end{equation}
    where $\calD_k$ are as defined in Lemma~\ref{lem:RS_balanced}. We will prove the stronger statement that no $\calF$-aware sampling algorithm $\calA$ can even achieve $\E[\mu\sim\calD]{\costTV_\calT(\calA;\mu)} \leq 1/4$.

    In order for $\costTV_\calT(\calA;\mu)$ to be finite, we must have $\leaves(\mu) = \leaves^{\le \calT(\mu)}(\mu)$. Henceforth, let $\leaves^* \coloneqq \leaves^{\le \calT(\calU)}(\calU)$. We must have
    \begin{equation}
        \TV\Bigl(\calU, \sum_{\ell\in\leaves^*} \Pr[\calA]{\ell \mid \calU}\cdot \nu_\ell\Bigr) \le 1/8\,, \label{eq:apply_pinsker}
    \end{equation}
    or else $\E[\mu\sim\calD]{\costTV_\calT(\calA;\mu)} \ge \frac{1}{2}\costTV_\calT(\calA;\calU) > 1/16$.

    For any leaf node $\ell$, let $v_1 \to w_1 \to v_2 \to\cdots \to w_{T-1} \to v_T$ denote the sequence of decision and leaf nodes along the root-to-leaf path to $\ell$, and suppose the edges $(v_i, w_i)$ are labeled with partial assignments $X_{S^{(i)}} = x^{(i)}$. If $\ell\in\leaves^*$, then the edges $(w_i, v_{i+1})$ are all labeled with $\mathrm{Unif}(\mathbb{F}_q)^{\otimes (n-|S^{(i)}|)}$. 
    
    Let $k_1 \le \cdots \le k_T$ denote the numbers $|S^{(1)}|, \ldots, |S^{(T)}|$ in sorted order. By Proposition~\ref{prop:independent}, for any MDS $\subspace$ of dimension $k > k_T$ we have $\Pr[\calA]{\ell \mid \calU_\subspace} = \Pr[\calA]{\ell \mid \calU}$. For $k_j < k < k_{j+1}$, by Lemma~\ref{lem:RS_balanced},
    \begin{equation}
        \Pr[\subspace\sim\calD_k]{\ell \ \text{avoids} \ \subspace} \ge 1 - \sum_{s>j} q^{-(k_s-k)} \ge 1 - T/q\,, \label{eq:avoid}
    \end{equation}
    and if $\ell$ avoids $\subspace$, the oracle's output under every query along the path is uniform marginals and again we have $\Pr[\calA]{\ell\mid\calU_\subspace} = \Pr[\calA]{\ell\mid \calU}$.
    The same reasoning applies to $k < k_1$.

    Let us write
    \begin{equation}
        \mathbb{E}_{\subspace\sim\calD_k}\TV\Bigl(\calU_\subspace , \sum_{\ell\in\leaves(\calU_\subspace)} \Pr[\calA]{\ell \mid \calU_\subspace}\cdot \nu_\ell\Bigr) \ge 1/2 - \mathbb{E}_{\subspace\sim\calD_k} \TV\Bigl(\calU, \sum_{\ell \in \leaves(\calU_\subspace)} \Pr[\calA]{\ell\mid\calU_\subspace}\cdot \nu_\ell\Bigr)
    \end{equation}
    where we used that $\TV(\calU, \calU_\subspace) \ge 1/2$ for any proper subspace $\subspace$. We can rewrite the mixture on the right-hand side as
    \begin{multline}
        \sum_{\ell \in \leaves^*:  \text{avoids} \ \subspace} \Pr[\calA]{\ell\mid \calU_\subspace}\cdot \nu_\ell + \sum_{\ell\in\leaves^*: \text{hits} \ \subspace}\Pr[\calA]{\ell\mid \calU_\subspace}\cdot \nu_\ell + \sum_{\ell \in \leaves(\calU_\subspace)\backslash\leaves^*}\Pr[\calA]{\ell\mid\calU_\subspace}\cdot \nu_\ell \\
        = \sum_{\ell\in\leaves^*} \Pr[\calA]{\ell\mid \calU}\cdot \nu_\ell - \sum_{\ell\in\leaves^*: \text{hits} \ \subspace} \Pr[\calA]{\ell\mid\calU}\cdot \nu_\ell + \sum_{\ell \in \leaves(\calU_\subspace)\backslash\leaves^*}\Pr[\calA]{\ell\mid\calU_\subspace}\cdot \nu_\ell\,, \label{eq:mixture}
    \end{multline}
    where we used that for $\ell\in\leaves^*$ that avoid $\subspace$, $\Pr[\calA]{\ell\mid\calU} = \Pr[\calA]{\ell\mid\calU_\subspace}$, and for $\ell \in \leaves^*$ that hit $\subspace$, it must be that $\Pr[\calA]{\ell\mid \calU_\subspace} = 0$ as the sampler under $\calU_\subspace$ must deviate from the path that leads to $\ell$. As $\sum_{\ell\in\leaves(\calU_\subspace)\backslash\leaves^*}\Pr[\calA]{\ell\mid\calU_\subspace} = \sum_{\ell\in\leaves^*:\text{hits} \subspace}\Pr[\calA]{\ell\mid\calU}$, the TV between $\calU$ and the mixture in Eq.~\eqref{eq:mixture} is thus upper bounded by $1/8 + \sum_{\ell\in\leaves^*: \text{hits} \ \subspace} \Pr[\calA]{\ell\mid \calU}$, and thus
    \begin{equation}
        \mathbb{E}_{\subspace\sim\calD_k}\TV\Bigl(\calU_\subspace , \sum_{\ell\in\leaves(\calU_\subspace)} \Pr[\calA]{\ell \mid \calU_\subspace}\cdot \nu_\ell\Bigr) \ge \frac{3}{8} - \sum_{\ell\in\leaves^*: \text{hits} \ \subspace} \Pr[\calA]{\ell\mid \calU}\,.
    \end{equation}

    We say that $\subspace$ is \emph{$\eta$-good} if it satisfies $\sum_{\ell\in \leaves^*: \text{hits} \subspace} \Pr[\calA]{\ell\mid \calU} \le \eta$ for some $\eta > 0$. Observe that 
    \begin{align}
        \frac{1}{n-1}\sum^{n-1}_{k=1} \Bigl\{\sum_{\ell\in\leaves^*}\Pr[\calA]{\ell\mid \calU} \cdot \Pr[\subspace\sim\calD_k]{\ell \ \text{hits} \ \subspace}\Bigr\} &= \sum_{\ell\in\leaves^*} \Pr[\calA]{\ell\mid \calU} \cdot \frac{1}{n-1}\sum^{n-1}_{k=1} \Pr[\subspace\sim\calD_k]{\ell \ \text{hits} \ \subspace} \\
        &\le \sum_{\ell\in\leaves^*} \Pr[\calA]{\ell \mid \calU} \cdot \frac{\calT(\calU) + (n - 1 - \calT(\calU)) \calT(\calU)/q}{n-1} \\
        &= \frac{\calT(\calU) + (n - 1 - \calT(\calU)) \calT(\calU)/q}{n-1} \\
        &\le \frac{2\calT(\calU)}{n - 1}
    \end{align}
    where in the second step we used that for any leaf $\ell$ at distance $2T$ from the root, there are at most $T$ dimensions $0 < k < n$ that are equal to the size of some partial assignment along the root-to-left path to $\ell$, and for all other dimensions $k$, $\Pr[\subspace\sim \calD_k]{\ell \ \text{hits} \ \subspace} \le T/q$ by Eq.~\eqref{eq:avoid}. By Markov's inequality, we conclude that for $\eta \coloneqq \frac{4\calT(\calU)}{n - 1} \ll 1$,
    \begin{equation}
        \Pr[0<k<n, \subspace\sim\calD_k]{\subspace \ \text{is} \ \eta\text{-good}} \ge 1/2 \,.
    \end{equation}
    We conclude that
    \begin{equation}
        \E[\mu\sim\calD]{\costTV_\calT(\calA;\mu)} \ge \frac{1}{2}\cdot \Pr[0<k<n, \subspace\sim\calD_k]{\subspace \ \text{is} \ \eta\text{-good}}\cdot \Bigl(\frac{3}{8} - \eta\Bigr) \ge \frac{1}{16}
    \end{equation}
    as claimed.
\end{proof}

\noindent In fact one sees from the definition of $\eta$ in the proof above that we have shown the even stronger statement that it is necessary to set the budget $\calT(\calU)$ for the uniform distribution to be \emph{linear} in $n$ for the costs to be sufficiently bounded across all $\mu\in\calF$. Intuitively, this comes from the fact that one has to make $\Omega(n)$ queries before one can decisively rule out that $\mu$ is supported on a subspace.

\subsection{Lower bounds for arbitrary information curves}
\label{sec:elevate}

Although the lower bound in Section~\ref{sec:warmup} is quite tailored to distributions over MDS codes, it turns out that the same idea can be extended to show that \emph{any} information curve admits a realization by some distribution which cannot be distinguished from a distribution with the same information curve except shifted upwards by an additive constant for all indices past a certain point.

\begin{theorem}\label{thm:elevate}
    Let $\bZ$ be the information curve associated to some distribution with total correlation $\TC$ and dual total correlation $\DTC$. Suppose that 
    \begin{equation}
        \min(\TC,\DTC)\log n \ll n\,. \label{eq:assumesmall}
    \end{equation}
    Given $1\le k < n$, let $\bZ^{\uparrow k}$ denote the information curve given by
    \begin{equation}
        Z^{\uparrow k}_j \coloneqq \begin{cases}
            Z_j & \text{if} \ j \le k \\
            Z_j + \log_2(q) & \text{if} \ j > k
        \end{cases}
    \end{equation}
    There exist a family $\calF$ of distributions whose information curves are from among $\{\bZ,\bZ^{\uparrow 1},\ldots,\bZ^{\uparrow n-1}\}$ such that for every $k$ there is at least one such distribution in $\calF$, and furthermore for any budget $\calT$ satisfying $\calT(\mu) \lesssim \max(1,\min(\TC(\mu), \DTC(\mu)))\log(n)$ for all $\mu\in\calF$, no $\calF$-aware sampling algorithm $\calA$ can achieve $\sup_{\mu\in\calF}\costTV_{\calT}(\calA;\mu) \le 1/16$.
\end{theorem}

\begin{proof}
    Let $\mu^*$ denote a distribution over $\Sigma^n$ with information curve $\bZ$, and let $\calU$ and $\calU_\subspace$ denote the uniform distribution over $\mathbb{F}^n_q$ and the uniform distribution over subspace $\subspace\subset\mathbb{F}^n_q$ as before, where $\subspace$ will range over RS codes. Define $\mu^*[\calU] \triangleq \mu^*\times \calU$ and $\mu^*[\calU_\subspace] \triangleq \mu^*\times \calU_\subspace$, regarded as distributions over $(\Sigma\times\mathbb{F}_q)^n$. If $\subspace$ has dimension $k$, then by the linearity of the information curve in the average entropy curve, and by additivity of entropy, 
    \begin{equation}
        \bZ(\mu^*[\calU]) = \bZ(\mu^*) + \bZ(\calU) = \bZ
    \end{equation}
    \begin{equation}
        \bZ(\mu^*[\calU_\subspace]) = \bZ(\mu^*) + \bZ(\calU_\subspace) = \bZ(\mu^*) + (\mathds{I}[j > k)_j = \bZ^{\uparrow k}\,.
    \end{equation}
    So to construct $\calF$, we include $\mu^*[\calU]$, and then for every dimension $1\le k < n$ and every $k$-dimensional RS code $\subspace$, we include $\mu^*[\calU_\subspace]$, thus satisfying the first condition in the Theorem.

    The rest of the proof is nearly identical to that of Theorem~\ref{thm:warmup}, and we defer it to Appendix~\ref{app:elevate}.    
\end{proof}

\section{Upper bound in terms of (dual) total correlation: proof of Theorem \ref{thm:tcdtc}}\label{upper_bounds}

In this section, we use Theorem \ref{main_result} to obtain data-agnostic bounds for the expected KL error of the fixed and random unmasking algorithms. As mentioned in Section~\ref{sec:connections}, Theorem~\ref{main_result} already immediately implies bounds from the prior work of~\cite{li2025convergence} and~\cite{austin2020multi}. In this section we use Theorem~\ref{main_result} to improve these bounds in most regimes, assuming only access to estimates $\widehat{\TC}$ and $\widehat{\DTC}$ of $\TC$ and $\DTC$. 

Recall that Theorem \ref{thm:tcdtc} states the existence of an algorithm attaining error at most $\epsilon$ and query complexity \begin{equation} k \leq 2 + (1 + \log n) \cdot (1 + \lceil \widehat{\TC} / \epsilon\rceil) \qquad (\text{resp.} \ k\le 2 + (1 + \log n) \cdot (1 + \lceil \widehat{\DTC} / \epsilon\rceil)\,.\end{equation} Provided that $\widehat{\TC}$ and $\widehat{\DTC}$ are constant-factor approximations of their respective estimands, this yields a query complexity proportional to $\min(\TC, \DTC)$ and is generally significantly better than Theorems \ref{austin_for_sampling} and \ref{gen_li_result}. 
\medskip

We now turn to the main technical content of this section, namely proving this result.

\begin{proof}[Proof of Theorem \ref{thm:tcdtc}]
We split into two cases, which are roughly similar. The main idea is to use an exponentially increasing schedule to attain the $\widehat\DTC$ bound and an exponentially decreasing schedule for the $\widehat\TC$ bound. This will attain the correct query complexity. Moreover, using the pictorial representation, we find that the horizontal slices of the error can be enlarged by a factor of $\frac {\widehat\DTC} \epsilon$ or $\frac {\widehat\TC} \epsilon$, respectively, and subsequently shifted horizontally to fit above or below the information curve, respectively. From this it follows that the total error is at most a factor of $\frac \epsilon {\widehat\DTC}$ or $\frac \epsilon {\widehat\TC}$ times either the area $\DTC$ or the area $\TC$, respectively, yielding the upper bound of $\epsilon$, provided that $\TC \leq \widehat \TC$ and $\DTC \leq \widehat \DTC$. We provide the full details below.

\medskip

\textbf{1. The $\widehat{\TC}$ bound}. We proceed by defining the mask schedule, and then analyzing the query complexity and sampling error. 

\medskip

\noindent \textbf{Mask Schedule}. We first define our mask schedule. Let $\zeta = 1 + \left \lceil \frac {\widehat\TC} \epsilon \right\rceil > 1$. If $\zeta \geq n+1$, then pick $k = n$ and $s_i = 1$ for all $i$; the sampler is perfect and the query complexity is $n \leq \left\lceil 1 +  \frac{\widehat\TC} \epsilon\right\rceil$, resolving this special case. From now on assume $\zeta \leq n$. Consider the sequence $N_i$ given by $N_0 = 0$ and then recursively $$N_i = \left\lfloor N_{i-1}+ (n - N_{i-1})\frac 1\zeta\right\rfloor$$ for $1 \leq i \leq \left\lfloor \frac{\log (n-\zeta+1)}{\log \frac{1}{1 - \frac 1\zeta}}\right\rfloor + 2 = \lambda$. 

Note that definitionally we have $N_i \geq N_{i-1}$ and by induction $N_i \leq n-1$ for all $i$. Moreover, $$N_i \geq N_{i-1} + (n-  N_{i-1})\frac 1\zeta - \frac{\zeta - 1}{\zeta} = (N_{i-1}-1)\left(1 - \frac 1 \zeta\right) + n\frac 1\zeta\,,$$
so that
$$n - \zeta + 1 - N_i \leq (n -\zeta + 1 - N_{i-1})\left(1 - \frac 1\zeta\right).$$ It follows that $$n-1 \geq N_\lambda \geq (n - \zeta + 1)\left(1 - \left(1 - \frac 1\zeta\right)^\lambda\right) > (n-\zeta+1)\left(1 - \frac{1}{n - \zeta+1}\right) = n-\zeta.$$

Now set $N_i = N_{i-1}+ 1$ for $\lambda + 1 \leq i \leq \lambda + n - N_\lambda$. Note that $N_{\lambda+ n - N_\lambda} = n$. Lastly, define $s_i = N_i - N_{i-1}$.\footnote{Note that potentially some of the final values of $s_i$ will be 0.} We consider the mask schedule given by $\{s_i\}_1^{\lambda + n - N_{\lambda}}$.

\noindent \textbf{Query Complexity}. The query complexity $k$ equals the number of steps of unmasking, i.e. $$k = \lambda + n - N_\lambda \leq \zeta + 2 + \frac{\log n}{\log \frac{1}{1 - \frac 1\zeta}} \leq 2 + \zeta(1 + \log n) \leq 2 + (1 + \log n) \cdot \left[1 + \left\lceil \frac {\widehat{\TC}} \epsilon\right\rceil\right],$$ where we have used the fact that $\log \frac{1}{1 - z} = -\log (1-z) \geq z$ for $z = \frac 1\zeta \in [0,1)$. 

\noindent \textbf{Sampling Error}. First observe that for $1 \leq i \leq \lambda$, we have $$s_i = N_i - N_{i-1} \leq (n - N_{i-1}) \frac 1\zeta \implies s_i \leq (n - N_i) \frac1{\zeta-1} \leq \frac \epsilon {\widehat\TC} (n - N_i).$$ Applying Theorem \ref{main_result}, we find that \allowdisplaybreaks
\begin{align*}
    \KL \mu \nu &\leq \mathbb E_{S_1, \ldots, S_k} [\KL \mu {\nu^{S_1, \ldots, S_k}}] \\ &= \sum_{i=1}^k \left[\left( \sum_{j=1}^{s_i} Z_{N_{i-1}+j}\right) - s_i Z_{N_{i-1}+1}\right] \\ &\leq \sum_{i=1}^k s_i \left(Z_{N_i} - Z_{N_{i-1} + 1}\right) \\ &\leq \frac{\epsilon}{\widehat{\TC}}\left(\sum_{i=1}^\lambda  (n-N_i) (Z_{N_i} - Z_{N_{i-1}}) \right)+ \sum_{i=\lambda+1}^k s_i (Z_{N_i} - Z_{N_{i-1}+1}) \\ &\leq \frac \epsilon{\widehat{\TC}}\left(\sum_{i=1}^{\lambda-1} (N_{i+1} - N_i)Z_{N_i}\right) + \frac \epsilon {\widehat{\TC}}(n - N_\lambda) Z_{N_\lambda} \\ &\leq \frac \epsilon {\widehat{\TC}} \left(\sum_{i=1}^{\lambda - 1} \sum_{j=N_i}^{N_{i+1}-1} Z_j \right) + \frac \epsilon {\widehat{\TC}} \sum_{j = N_\lambda}^{n - 1} Z_j \\ &= \frac \epsilon {\widehat{\TC}} \sum_{j=1}^{n-1} Z_j \\ &\leq \frac \epsilon {\widehat{\TC}} \cdot \TC \\ &\leq \epsilon,
\end{align*} where we let $Z_0 = Z_1 =0$ by convention and we have repeatedly used that the $Z_j$'s are nonnegative and nondecreasing (see Lemma \ref{tc_dtc_facts}). Note that the fourth line follows from a rearrangement and the fact that $N_i = N_{i-1} + 1$ for $i > \lambda$. Thus the algorithm yields the correct query complexity and sampling error, completing the proof of this case. 

\medskip

\textbf{2. The $\widehat\DTC$ bound}. We proceed in the same three steps as in case 1; the proof will be largely similar, except that the mask schedule is essentially flipped.
\medskip

\noindent \textbf{Mask Schedule}. Let $\zeta = 1 + \left \lceil \frac {\widehat\DTC} \epsilon \right\rceil > 1$. If $\zeta \geq n+1$, then pick $k = n$ and $s_i = 1$ for all $i$; the sampler is perfect and the query complexity is $n \leq \left\lceil 1 +  \frac\DTC \epsilon\right\rceil$, resolving this special case. From now on assume $\zeta \leq n$. Consider the sequence $N_i'$ given by $N_0' = n$ and then recursively $$N_i' = \left\lceil N_{i-1}'\left(1 - \frac 1 \zeta\right)\right\rceil$$ for $1 \leq i \leq \left\lfloor \frac{\log (n-\zeta+1)}{\log \frac{1}{1 - \frac 1\zeta}}\right\rfloor + 2 = \lambda$. 

Note that definitionally we have $N_i' \leq N_{i-1}'$ and by induction $N_i' \geq 1$ for all $i$. Moreover, $$N_i' \leq N_{i-1}'\left(1 - \frac 1\zeta\right) + \frac{\zeta - 1}{\zeta} \implies N_i' - \zeta+1 \leq \left(N_{i-1}' - \zeta+1\right)\left(1 - \frac 1\zeta\right).$$ It follows that $$1 \leq N_\lambda' \leq \zeta-1 + (n - \zeta+1)\left(1 - \frac 1\zeta\right)^\lambda < \zeta-1 +(n-\zeta + 1)\frac{1}{n - \zeta + 1} = \zeta.$$ Now, set $N_i' = N_{i-1}' - 1$ for $\lambda+ 1 \leq i \leq \lambda + N_\lambda'$. Note that $N_{\lambda + N_\lambda'}' = 0$. Lastly, define $s_i = N_{\lambda + N_\lambda' - i}' - N_{\lambda + N_\lambda' - i+1}'$. We consider the mask schedule given by $\{s_i\}_1^{\lambda + N_\lambda'}$. 

\noindent \textbf{Query Complexity}. The query complexity $k$ equals the number of steps of unmasking, i.e. $$k = \lambda + N_\lambda' \leq \zeta + 2 +  \frac{\log n}{\log \frac{1}{1 - \frac 1\zeta}} \leq 2 + \zeta(1 + \log n) \leq 2 + (1 + \log n) \cdot \left[1 + \left\lceil \frac {\widehat\DTC} \epsilon\right\rceil\right],$$ where we have used the fact that $\log \frac{1}{1 - z} = -\log (1-z) \geq z$ for $z = \frac 1\zeta \in [0,1)$.  

\noindent \textbf{Sampling Error}. First observe that for $i > N_\lambda'$, we have $\lambda + N_\lambda ' - i + 1 \leq \lambda$ and hence $$s_i = N_{\lambda + N_\lambda' - i}'- N_{\lambda + N_\lambda' - i+1}' \leq N_{\lambda + N_\lambda' -i}' \frac 1\zeta \implies s_i \leq N_{\lambda + N_\lambda' - i+1}' \frac1{\zeta-1} \leq \frac \epsilon {\widehat\DTC} N_{\lambda + N_\lambda' - i+1}'.$$ As usual, let $$N_i = \sum_{j \leq i} s_i = N_{\lambda + N_\lambda' - i}'.$$ Applying Theorem \ref{main_result}, we find that \allowdisplaybreaks
 \begin{align*}
    \KL \mu \nu &\leq \mathbb E_{S_1, \ldots, S_k} [\KL \mu {\nu^{S_1, \ldots, S_k}}] \\ &= \sum_{i=1}^k \left[\left( \sum_{j=1}^{s_i} Z_{N_{i-1}+j}\right) - s_i Z_{N_{i-1}+1}\right] \\ &\leq \sum_{i=1}^k s_i(Z_{N_i} - Z_{N_{i-1}+1}) \\ &\leq \frac \epsilon {\widehat\DTC} \left(\sum_{i=N_\lambda'+1}^{N_\lambda' + \lambda} N_{\lambda + N_\lambda' - i + 1}' (Z_{N_i} - Z_{N_{i-1}})\right) + \sum_{i \leq N_\lambda'} s_i (Z_{N_i} - Z_{N_{i-1}+1}) \\ &= \frac \epsilon {\widehat\DTC} \left(N_{N_\lambda' + \lambda - 1}Z_n\right) - \frac \epsilon {\widehat\DTC} \left(\sum_{i = N_{\lambda' + 1}}^{N_\lambda' + \lambda - 1} (N_i - N_{i-1}) Z_{N_i}\right) - \frac \epsilon {\widehat\DTC} \left(N_{N_\lambda'}Z_{N_{N_{\lambda}'}}\right) \\ &\leq \frac \epsilon {\widehat\DTC} \left(N_{N_\lambda' + \lambda - 1}Z_n\right) - \frac \epsilon {\widehat\DTC} \left(\sum_{i = N_{\lambda' + 1}}^{N_\lambda' + \lambda - 1} \left(\sum_{j = N_{i-1}+1}^{N_i}Z_j \right)\right) - \frac \epsilon {\widehat\DTC} \left(\sum_{j=1}^{N_{N_{\lambda}'}} Z_j\right)\\ &= \frac \epsilon {\widehat\DTC} \left(\sum_{j = 1}^{N_{N_\lambda' + \lambda - 1}} (Z_n - Z_j)\right) \\ &\leq \frac \epsilon {\widehat\DTC} \sum_{j=1}^n (Z_n - Z_j) \\ &= \frac \epsilon {\widehat\DTC} \cdot \DTC \\ &\leq \epsilon, \end{align*} where we let $Z_0 = Z_1 =0$ by convention and we have repeatedly used that the $Z_j$'s are nonnegative and nondecreasing (see Lemma \ref{tc_dtc_facts}). Note that the fourth line follows from a rearrangement and the fact that $N_i = N_{i-1} + 1$ for $i \leq N_\lambda'$. Thus the algorithm yields the correct query complexity and sampling error, completing the proof of this case.

\medskip

In conclusion, we find that in both cases there exists a mask schedule with the desired query complexity and sampling error. This completes the proof. 
\end{proof}


\noindent\textbf{Knowledge of $\TC$ and $\DTC$.} We elaborate briefly on the ``hyperparameter sweep'' discussed in the introduction. While the above result does not require knowledge of the data distribution or the entire information curve, it nonetheless requires the values of $\TC$ and $\DTC$. These values are in general unknown and moreover not readily estimable from our conditional oracle. In practice, however, we can treat $\TC$ and $\DTC$ as sampling hyperparameters and sweep over a feasible range $\calH$. 

We suggest a choice of $\calH$ as follows. First, it is not difficult to see that if we choose estimates $\widehat{\TC} \in [\TC, 2\cdot \TC]$ and $\widehat \DTC \in [\DTC, 2\cdot \DTC]$, the mask schedule in the proof of Theorem \ref{thm:tcdtc} achieves error at most $\epsilon$ and query complexity within a factor of two of $$2 + (1 + \log n) \cdot \left(1 + \min\left(\left\lceil \frac{\TC}{\epsilon}\right\rceil, \left\lceil \frac \DTC \epsilon\right\rceil\right)\right),$$ that is the complexity if we had complete knowledge of $\TC$ and $\DTC$.  Moreover, we know that $$1 \leq  \left\lceil\frac \DTC\epsilon\right\rceil, \left\lceil\frac \TC\epsilon\right\rceil \hspace{0.2cm} \mathrm{and} \hspace{0.2cm} \DTC, \TC \leq nZ_n \leq nH_1 \leq n \log |\Sigma|.$$ Combining these observations, we can take $$\calH = \{2^i: i \in \mathbb Z, \epsilon \leq i \leq n\log |\Sigma|\}; \hspace{0.2cm} |\calH| = \calO\left(\log \frac{n \log |\Sigma|}{\epsilon}\right),$$ for which there exists $(\widehat \TC, \widehat \DTC) \in \calH^2$ which if used as the estimates of $\TC$ and $\DTC$ yields the desired error and query complexity. Under this choice of $\calH$, the hyperparameter sweep incurs an extra query complexity factor of $|\calH|^2 = \calO\left(\log \left(\frac{n \log |\Sigma|}{\epsilon}\right)^2\right)$. For most choices of $n, \epsilon, |\Sigma|$, this will be a polylogarithmic factor in $n$.

\bibliographystyle{alpha}
\bibliography{refs}

@article{li2024sharp,
	title={A sharp convergence theory for the probability flow {ODEs} of diffusion models},
	author={Li, Gen and Wei, Yuting and Chi, Yuejie and Chen, Yuxin},
	journal={arXiv preprint arXiv:2408.02320},
	year={2024}
}

@article{reeves2025information,
  title={Information-Theoretic Proofs for Diffusion Sampling},
  author={Reeves, Galen and Pfister, Henry D},
  journal={arXiv preprint arXiv:2502.02305},
  year={2025}
}

@inproceedings{lee2023convergence,
  title={Convergence of score-based generative modeling for general data distributions},
  author={Lee, Holden and Lu, Jianfeng and Tan, Yixin},
  booktitle={International Conference on Algorithmic Learning Theory},
  pages={946--985},
  year={2023},
  organization={PMLR}
}

@inproceedings{chen2023improved,
  title={Improved analysis of score-based generative modeling: user-friendly bounds under minimal smoothness assumptions},
  author={Chen, Hongrui and Lee, Holden and Lu, Jianfeng},
  booktitle={International Conference on Machine Learning},
  pages={4735--4763},
  year={2023},
  organization={PMLR}
}

@article{conforti2025kl,
  title={{KL} convergence guarantees for score diffusion models under minimal data assumptions},
  author={Conforti, Giovanni and Durmus, Alain and Silveri, Marta Gentiloni},
  journal={SIAM Journal on Mathematics of Data Science},
  volume={7},
  number={1},
  pages={86--109},
  year={2025},
  publisher={SIAM}
}

@inproceedings{Ben+24Diffusion,
title={Nearly {$d$}-linear convergence bounds for diffusion models via stochastic localization},
author={Joe Benton and Valentin De Bortoli and Arnaud Doucet and George Deligiannidis},
booktitle={The Twelfth International Conference on Learning Representations},
year={2024},
}

@inproceedings{Chen+23SGM,
title={Sampling is as easy as learning the score: theory for diffusion models with minimal data assumptions},
author={Sitan Chen and Sinho Chewi and Jerry Li and Yuanzhi Li and Adil Salim and Anru R. Zhang},
booktitle={The Eleventh International Conference on Learning Representations },
year={2023},
}

@article{jerrum1986random,
  title={Random generation of combinatorial structures from a uniform distribution},
  author={Jerrum, Mark R and Valiant, Leslie G and Vazirani, Vijay V},
  journal={Theoretical computer science},
  volume={43},
  pages={169--188},
  year={1986},
  publisher={Elsevier}
}

@article{andrea2008estimating,
  title={Estimating random variables from random sparse observations},
  author={Andrea, Montanari},
  journal={European Transactions on Telecommunications},
  volume={19},
  number={4},
  pages={385--403},
  year={2008},
  publisher={Wiley Online Library}
}

@article{el2022information,
  title={An information-theoretic view of stochastic localization},
  author={El Alaoui, Ahmed and Montanari, Andrea},
  journal={IEEE Transactions on Information Theory},
  volume={68},
  number={11},
  pages={7423--7426},
  year={2022},
  publisher={IEEE}
}

@inproceedings{yoshida2014approximation,
  title={Approximation schemes via Sherali-Adams hierarchy for dense constraint satisfaction problems and assignment problems},
  author={Yoshida, Yuichi and Zhou, Yuan},
  booktitle={Proceedings of the 5th conference on Innovations in theoretical computer science},
  pages={423--438},
  year={2014}
}

@inproceedings{manurangsi2017birthday,
  title={A Birthday Repetition Theorem and Complexity of Approximating Dense CSPs},
  author={Manurangsi, Pasin and Raghavendra, Prasad},
  booktitle={44th International Colloquium on Automata, Languages, and Programming (ICALP 2017)},
  pages={78--1},
  year={2017},
  organization={Schloss Dagstuhl--Leibniz-Zentrum f{\"u}r Informatik}
}

@inproceedings{raghavendra2012approximating,
  title={Approximating CSPs with global cardinality constraints using SDP hierarchies},
  author={Raghavendra, Prasad and Tan, Ning},
  booktitle={Proceedings of the twenty-third annual ACM-SIAM symposium on Discrete Algorithms},
  pages={373--387},
  year={2012},
  organization={SIAM}
}

@inproceedings{jain2019mean,
  title={Mean-field approximation, convex hierarchies, and the optimality of correlation rounding: a unified perspective},
  author={Jain, Vishesh and Koehler, Frederic and Risteski, Andrej},
  booktitle={Proceedings of the 51st Annual ACM SIGACT Symposium on Theory of Computing},
  pages={1226--1236},
  year={2019}
}

@article{eldan2018exponential,
  title={Exponential random graphs behave like mixtures of stochastic block models},
  author={Eldan, Ronen and Gross, Renan},
  journal={The Annals of Applied Probability},
  volume={28},
  number={6},
  pages={3698--3735},
  year={2018},
  publisher={JSTOR}
}

@article{eldan2018gaussian,
  title={Gaussian-width gradient complexity, reverse log-Sobolev inequalities and nonlinear large deviations},
  author={Eldan, Ronen},
  journal={Geometric and Functional Analysis},
  volume={28},
  number={6},
  pages={1548--1596},
  year={2018},
  publisher={Springer}
}

@article{khanna2025mercury,
  title={Mercury: Ultra-fast language models based on diffusion},
  author={Khanna, Samar and Kharbanda, Siddhant and Li, Shufan and Varma, Harshit and Wang, Eric and Birnbaum, Sawyer and Luo, Ziyang and Miraoui, Yanis and Palrecha, Akash and Ermon, Stefano and others},
  journal={arXiv preprint arXiv:2506.17298},
  year={2025}
}

@article{nie2025large,
  title={Large language diffusion models},
  author={Nie, Shen and Zhu, Fengqi and You, Zebin and Zhang, Xiaolu and Ou, Jingyang and Hu, Jun and Zhou, Jun and Lin, Yankai and Wen, Ji-Rong and Li, Chongxuan},
  journal={arXiv preprint arXiv:2502.09992},
  year={2025}
}

@inproceedings{canonne2021random,
  title={Random restrictions of high dimensional distributions and uniformity testing with subcube conditioning},
  author={Canonne, Cl{\'e}ment L and Chen, Xi and Kamath, Gautam and Levi, Amit and Waingarten, Erik},
  booktitle={Proceedings of the 2021 ACM-SIAM Symposium on Discrete Algorithms (SODA)},
  pages={321--336},
  year={2021},
  organization={SIAM}
}

@article{ren2025fast,
  title={Fast solvers for discrete diffusion models: Theory and applications of high-order algorithms},
  author={Ren, Yinuo and Chen, Haoxuan and Zhu, Yuchen and Guo, Wei and Chen, Yongxin and Rotskoff, Grant M and Tao, Molei and Ying, Lexing},
  journal={arXiv preprint arXiv:2502.00234},
  year={2025}
}

@article{chen2024convergence,
  title={Convergence analysis of discrete diffusion model: Exact implementation through uniformization},
  author={Chen, Hongrui and Ying, Lexing},
  journal={arXiv preprint arXiv:2402.08095},
  year={2024}
}

@article{canonne2020survey,
  title={A survey on distribution testing: Your data is big. But is it blue?},
  author={Canonne, Cl{\'e}ment L},
  journal={Theory of Computing},
  pages={1--100},
  year={2020},
  publisher={Theory of Computing Exchange}
}

@article{canonne2015testing,
  title={Testing probability distributions using conditional samples},
  author={Canonne, Cl{\'e}ment L and Ron, Dana and Servedio, Rocco A},
  journal={SIAM Journal on Computing},
  volume={44},
  number={3},
  pages={540--616},
  year={2015},
  publisher={SIAM}
}

@inproceedings{chakraborty2013power,
  title={On the power of conditional samples in distribution testing},
  author={Chakraborty, Sourav and Fischer, Eldar and Goldhirsh, Yonatan and Matsliah, Arie},
  booktitle={Proceedings of the 4th conference on Innovations in Theoretical Computer Science},
  pages={561--580},
  year={2013}
}

@inproceedings{anari2024parallel,
  title={Parallel sampling via counting},
  author={Anari, Nima and Gao, Ruiquan and Rubinstein, Aviad},
  booktitle={Proceedings of the 56th Annual ACM Symposium on Theory of Computing},
  pages={537--548},
  year={2024}
}

@article{chatterjee2016nonlinear,
  title={Nonlinear large deviations},
  author={Chatterjee, Sourav and Dembo, Amir},
  journal={Advances in Mathematics},
  volume={299},
  pages={396--450},
  year={2016},
  publisher={Elsevier}
}

@article{lavenant2025error,
  title={Error Bounds and Optimal Schedules for Masked Diffusions with Factorized Approximations},
  author={Lavenant, Hugo and Zanella, Giacomo},
  journal={arXiv preprint arXiv:2510.25544},
  year={2025}
}

@book{wupolyanskiy,
  title={Information theory: From coding to learning},
  author={Polyanskiy, Yury and Wu, Yihong},
  year={2025},
  publisher={Cambridge university press}
}

@article{eldan2018decomposition,
  title={Decomposition of mean-field gibbs distributions into product measures},
  author={Eldan, Ronen and Gross, Renan},
  journal={Electronic Journal of Probability},
  volume={23},
  year={2018}
}

@article{austin2019structure,
  title={The Structure of Low-Complexity Gibbs Measures on Product Spaces},
  author={Austin, Tim},
  journal={The Annals of Probability},
  volume={47},
  number={6},
  pages={4002--4023},
  year={2019}
}

@article{kang2025parallelbench,
  title={ParallelBench: Understanding the Trade-offs of Parallel Decoding in Diffusion LLMs},
  author={Kang, Wonjun and Galim, Kevin and Oh, Seunghyuk and Lee, Minjae and Zeng, Yuchen and Zhang, Shuibai and Hooper, Coleman and Hu, Yuezhou and Koo, Hyung Il and Cho, Nam Ik and others},
  journal={arXiv preprint arXiv:2510.04767},
  year={2025}
}

@article{shi2024simplified,
  title={Simplified and generalized masked diffusion for discrete data},
  author={Shi, Jiaxin and Han, Kehang and Wang, Zhe and Doucet, Arnaud and Titsias, Michalis},
  journal={Advances in neural information processing systems},
  volume={37},
  pages={103131--103167},
  year={2024}
}

@article{sahoo2024simple,
  title={Simple and effective masked diffusion language models},
  author={Sahoo, Subham and Arriola, Marianne and Schiff, Yair and Gokaslan, Aaron and Marroquin, Edgar and Chiu, Justin and Rush, Alexander and Kuleshov, Volodymyr},
  journal={Advances in Neural Information Processing Systems},
  volume={37},
  pages={130136--130184},
  year={2024}
}

@inproceedings{lou2024discrete,
  title={Discrete diffusion modeling by estimating the ratios of the data distribution},
  author={Lou, Aaron and Meng, Chenlin and Ermon, Stefano},
  booktitle={Proceedings of the 41st International Conference on Machine Learning},
  pages={32819--32848},
  year={2024}
}

@inproceedings{chang2022maskgit,
  title={Maskgit: Masked generative image transformer},
  author={Chang, Huiwen and Zhang, Han and Jiang, Lu and Liu, Ce and Freeman, William T},
  booktitle={Proceedings of the IEEE/CVF conference on computer vision and pattern recognition},
  pages={11315--11325},
  year={2022}
}

@article{li2025convergence,
  title={A Convergence Theory for Diffusion Language Models: An Information-Theoretic Perspective},
  author={Li, Gen and Cai, Changxiao},
  journal={arXiv preprint arXiv:2505.21400},
  year={2025}
}

@article{austin2020multi,
  title={Multi-variate correlation and mixtures of product measures},
  author={Austin, Tim},
  journal={Kybernetika},
  volume={56},
  number={3},
  pages={459--499},
  year={2020},
  publisher={Institute of Information Theory and Automation AS CR}
}

@article{birge1987estimating,
  title={Estimating a density under order restrictions: Nonasymptotic minimax risk},
  author={Birg{\'e}, Lucien},
  journal={The Annals of Statistics},
  pages={995--1012},
  year={1987},
  publisher={JSTOR}
}

@inproceedings{chen2009average,
  title={Average entropy functions},
  author={Chen, Qi and He, Chen and Jiang, Lingge and Wang, Qingchuan},
  booktitle={2009 IEEE International Symposium on Information Theory},
  pages={2632--2633},
  year={2009},
  organization={IEEE}
}

\newpage

\appendix

\section{Logarithmic overhead is necessary} \label{sec:logn}

In this section, we show that the $\log n$ term in Theorem \ref{thm:tcdtc} is necessary. This is essentially implicit in~\cite{birge1987estimating} and has been used in various works on monotone distribution estimation. We were, however, not able to find a complete statement and proof in the literature and provide one for completeness. First, recall the following definition. 

\begin{definition}[Piecewise Functions]
    A \textit{$k$-piecewise} function $f: [n] \rightarrow R$ is a function for which there are \textit{at most} $k$ values $i \in [n-1]$ such that $f(i) \neq f(i+1)$. 
\end{definition}

We can now state the main result of this subsection. 

\begin{lemma}\label{lem:logarithmic}
    Let $n \geq 2$, $\epsilon$ be such that $\frac 2n \log \frac 2\epsilon \leq \epsilon \leq \frac{1}{\log n}$, and $k \leq c \cdot \tfrac{\log n}{\epsilon}$ for some sufficiently small constant $c$. Then there exists a non-negative monotone increasing function $f: [n] \rightarrow \R$ satisfying $\sum_i f(i) = 1$ such that for any function $h: [n] \rightarrow \R$ which is a $k$-piecewise constant function, we have that $$\norm{f - g}_{L^1} \geq \Omega (\epsilon).$$
\end{lemma}
\begin{proof} We define $f$ as follows. For $i = 0, \ldots, \left\lceil\tfrac{\log (n+1)}{\log (1 + \epsilon)} - 1\right\rceil= m \leq \frac{\log (n+1)}{\epsilon}$, we let $B_i = \{\lfloor (1 + \epsilon)^i \rfloor, \ldots, \min(\lfloor (1 + \epsilon)^{i + 1}\rfloor - 1, n)\}$, and for $x \in B_i$ we let $$f(x) = p_i = \frac 14 \cdot \frac{(1 + \epsilon)^{-i}}{ \log n}.$$ Note that $$\sum_{i=1}^n f(x) \leq \frac 14 \sum_{i = 0, \epsilon(1+\epsilon)^i \geq 1}^m (\epsilon(1 + \epsilon)^i + 1)\frac{(1+\epsilon)^{-i}}{\log n} \leq \frac 14 \left(\sum_{i = 0}^m 2\frac \epsilon {\log n}\right) \leq 1$$ as $m \leq \frac{2 \log n}{\epsilon}$ and so this is a valid choice of $f$.
    
Let $h$ be any $k$-piecewise function, and let $I_1, \ldots, I_k$ denote the partition of $[n]$ into $k$ disjoint intervals on which $h$ is constant. We will refer to the right endpoints of these intervals as the \textit{breakpoints} of $h$. The remainder of the proof proceeds in two main steps. We first show that $h$ may be modified to have two desirable structural qualities: namely to have breakpoints only at the right endpoints of the $B_i$ and values only in the $\{p_i\}$. We then directly analyze functions $h$ with these properties to prove the bound. 
\medskip

\noindent \textbf{Shifting the image of $h$.} We first claim that for any $h$, there is some $h_{\mathrm{image}}$ such that $ \norm{h_{\mathrm{image}} - f}_{L^1} \leq \norm{h - f}_{L^1} $, and $h_{\mathrm{image}}$ has the same breakpoints as $h$ but $h_{\mathrm{image}}(x) = p_{i_j}$ for all $x \in I_j$, for values $i_j$ satisfying $B_{i_j} \cap I_j \neq \emptyset$. That is, on each interval where $h$ is constant, we can replace the value of $h$ on that interval with one of the values that $f$ attains on the same interval. 

Indeed, this is because the contribution of $I_j$ to the total error is $F_j(h(x)) = \sum_{x \in I_j} |h(x) - f(x)|$. Suppose $h'(x) \in [p_{t}, p_{t+1}]$. Then since $F_j(h'(x))$ is linear in $[p_t, p_{t+1}]$, it must be optimized at one of the endpoints. In particular, we may replace the value of $h$ on $I_j$ with either $p_t$ or $p_{t+1}$ without increasing the number of pieces of total error. Applying this result to all intervals $I_j$ yields some $h_{\mathrm{image}}$ which satisfies the conditions of the claim. 

\medskip

\noindent \textbf{Shifting the breakpoints of $h$.} We now claim that for any $h$, there is some $h_{\mathrm{final}}$ such that $\norm{h_{\mathrm{final}} - f}_{L^1} \leq \norm{h - f}_{L^1} $, which is also $k$-piecewise but whose breakpoints are a subset of the breakpoints of $f$, and whose image is contained in the image of $h$. 

Indeed, suppose $h$ has a breakpoint $t \in I_j \cap B_\ell$ which is not the right endpoint of $B_\ell$; call such breakpoints \textit{bad}. The contribution of $I_j$ to the total error is $F_j(h(x)) = \sum_{x \in I_j} |h(x) - f(x)|$. We now have two cases. 

\textbf{Case 1.} If $|h(t) - f(t) | > |h(t+1) - f(t+1)|$, let $h'(r) = h(t+1) \hspace{0.1cm} \forall \hspace{0.1cm} r \in I_j \cap B_\ell, r \leq t$, and $h'(x) = h(x)$ otherwise. Then the number of pieces and total error of $h'$ are not greater than those of $h$. Moreover, either the total number of pieces decreases by 1 or $h'$ has one less bad breakpoint than $h$.\footnote{The former occurs if $I_j$ has a smaller left endpoint than $B_\ell$ and the latter otherwise.}

\textbf{Case 2.} If $|h(t) - f(t) | \leq |h(t+1) - f(t+1)|$, let $h'(r) = h(t)\hspace{0.1cm}  \forall \hspace{0.1cm} r \in I_{j+1} \cap B_\ell, r > t$, and $h'(x) = h(x)$ otherwise. Then the number of pieces and total error of $h'$ are not greater than those of $h$. Moreover, either the total number of pieces decreases by 1 or $h'$ has one less bad breakpoint than $h$.\footnote{The former occurs if $I_{j+1}$ has a larger right endpoint than $B_\ell$ and the latter otherwise.}

Iterating the operation $h \rightarrow h'$, we conclude that after a finite number of rounds there are no bad breakpoints. The output $h_\mathrm{final}$ satisfies the conditions of the claim. 

\medskip

\noindent \textbf{Completing the proof.} Combining these claims, there exists \textit{an} optimal $k$-piecewise approximation $h$ to $f$ for which $I_j = \{\left\lfloor(1 + \epsilon)^{i_j}\right\rfloor + 1, \left\lfloor(1 + \epsilon)^{i_{j + 1}}\right\rfloor\}$ for some $0 = i_1\leq  \cdots \leq i_{k + 1} = m$, and for all $x \in I_j$, we have that $h(x) = p_{t_j}$ for some $t_j$. Moreover, breaking up each interval into the intervals before and after $t_j$, we may further assume that $h(x) = p_{i_j}$ or $h(x) = p_{i_{j + 1}}$.\footnote{Note that this operation at most doubles the number of intervals in $h$; thus, it can be accounted for by changing the constant $c$ appropriately.} Lastly, we may ignore the regions $B_j$ for which $j \lesssim \frac{\log (2/\epsilon)}{\epsilon}$, so that for all considered $B_s$, we have $\epsilon(1+\epsilon)^s \geq 2$.\footnote{This operation at most increments the number of intervals by a constant, and underestimates the total error, so it is also permissible.} Now, let $\ell_j = i_{j + 1} - i_j$. We have the following two cases. 

\noindent \textbf{Large interval regime.} First, if there is any interval $I_j$ such that $\ell_j \geq \tfrac{2 (1 + \epsilon)}{\epsilon}$, then if $h(x) = (1 + \epsilon)^{i_{j + 1}}$ on this interval, and so the error of the approximation on this interval is at least \begin{align*}
\frac{1}{\log n} \sum_{s = i_j}^{i_{j + 1}} \left| \frac{1}{(1 + \epsilon)^s} - \frac{1}{(1 + \epsilon)^{i_{j + 1}}}  \right| \cdot \frac12\epsilon(1 + \epsilon)^s &= \frac{\epsilon}{2\log n} \left( \ell_j - \sum_{s = 0}^{\ell_j - 1} \frac{1}{(1 + \epsilon)^s} \right) \\ &\geq \frac{\epsilon}{2\log n} \left( \ell_j - \frac{1 + \epsilon}{\epsilon} \right) \\ &\geq \tfrac{1}{2\log n} = \Omega (\epsilon). \end{align*} If instead $h(x) = (1 + \epsilon)^{i_{j}}$, we obtain the same result via an analogous calculation. In either case, the desired statement holds. 

\noindent \textbf{Small interval regime.} Otherwise, assume that $\ell_j \leq \tfrac{2 (1 + \epsilon)}{\epsilon}$ for all $j$. Recall the approximation inequality $1 - \tfrac{1}{(1 + \epsilon)^r} \geq \min (1/2, r \epsilon/2)$, which follows from Bernoulli's inequality. Applying it, we find that if $h(x) = (1 + \epsilon)^{i_{j + 1}}$, the error on $I_j$ is at least \begin{align*}
        \frac{1}{\log n} \sum_{s = i_j}^{i_{j + 1}} \left| \frac{1}{(1 + \epsilon)^s} - \frac{1}{(1 + \epsilon)^{i_{j + 1}}}  \right| \cdot \frac12 \epsilon(1+\epsilon)^s &\geq \frac{\epsilon}{4\log n} \sum_{s = 0}^{\ell_j - 1} s \epsilon \\ &\geq \frac{\epsilon^2}{16 \log n} (\ell_j-1)^2
\end{align*} If instead $h(x) = (1+\epsilon)^{i_j}$, we obtain the same result again via an analogous calculation. Thus, the overall error of the approximation can be lower bounded by $$ \left\| h - f \right\|_{L^1} \geq \frac{\epsilon^2}{16 \log n} \sum_{j = 1}^k (\ell_j-1)^2.$$ Since $\sum_{j = 1}^k (\ell_j-1) \geq (1-c)\tfrac{\log n}{\epsilon}$, by the Cauchy-Schwarz inequality we conclude that $$\|h-f\|_{L^1} \geq \frac{\epsilon^2(1-c)^2}{16 \log n} \cdot \frac{(\log n)^2}{k\epsilon^2} = \frac{(1-c)^2}{16} \cdot \frac{\log n}{k} \geq \Omega\left(\epsilon\right),$$ completing the proof of the desired claim.
\end{proof}

To make this result applicable to distributions, recall the following result about the realizability of any information curve by some distribution: 
\begin{lemma}[Theorem 1 in~\cite{chen2009average}]\label{lem:realizability}
    For any information curve $0 \leq Z_1 \leq \cdots \leq Z_n$ and approximation errors $\epsilon_1, \ldots, \epsilon_n > 0$, there exists a distribution $\mu$ for which $|Z_i(\mu) - Z_i| \leq \epsilon_i$ for all $i$. 
\end{lemma}

Combining Lemmas \ref{lem:logarithmic} and \ref{lem:realizability}, we conclude that there exist distributions for which the $\log n$ is unavoidable on many parameter regimes.

\begin{theorem}\label{thm:logaritihmic}
    Let $n \geq 2$, $\epsilon$ be such that $\frac 2n \log \frac 2\epsilon \leq \epsilon \leq \frac1{\log n}$, and $c$ be a sufficient small constant. Then there exists a distribution $\mu$ such that for any unmasking schedule $(k, \{s_i\}_1^k)$ with $k \leq c \cdot \frac{\log n} \epsilon$, we have $$\mathbb E_{S_1, \ldots, S_k} \left[\KL \mu {\nu^{S_1, \ldots, S_k}}\right] \geq \Omega(\epsilon),$$ where the expectation is taken over all partitions $S = \bigsqcup_{i=1}^k S_i$ and $\nu^{S_1, \ldots, S_k}$ denotes the distribution outputted by $\calA(k, \{S_i\}_1^k)$. 
\end{theorem}
\begin{proof}
    Applying Theorem \ref{thm:main}, the expected KL error is given by $\|\mathbf Z - \mathbf Z^{\mathbf N}\|_{L^1}$. The result follows from choosing a distribution $\mu$ via Lemma \ref{lem:realizability} with $\epsilon_i = o(\frac \epsilon n)$ and applying Lemma \ref{lem:logarithmic} and the triangle inequality for the $L^1$ norm.      
\end{proof}

\section{Recovering existing bounds}\label{sec:recover}

In this section we recover the iteration complexity bound of~\cite{li2025convergence} and an iteration complexity bound which is implicit in~\cite{austin2020multi}.

\subsection{Recovering the bound of Li and Cai~\cite{li2025convergence}}
\label{sec:licai}

In \cite{li2025convergence}, the authors prove the following bound on the sampling error, given the sizes of the mask schedule.

\begin{theorem}[Theorem 1 of \cite{li2025convergence}]\label{gen_li_result}
Let $\mu$ be the data distribution and $\nu^{S_1, \ldots, S_k}$ be the output of the fixed unmasking algorithm $\calA_{\mathrm{fixed}}(k, \{S_i\}_1^k)$. Let $s_{\mathrm{max}} = \max_{i=1}^k |S_i|$. Then we have \begin{equation}
\mathbb{E}_{S_1, \ldots, S_k}\left[\KL \mu {\nu^{S_1, \ldots, S_k}}\right] \leq \frac{2^{\lceil \log_2 s_{\mathrm{max}}\rceil } - 1}{n} \sum_{i=1}^n I\left(X_i; \{X_j\}_{j \neq i}\right) = \frac{2^{\lceil \log_2 s_{\mathrm{max}}\rceil } - 1}{n} (\TC + \DTC)
\end{equation}

\end{theorem}

We provide a short proof of this result via Theorem \ref{main_result}.
\begin{proof}
    The equality in the theorem follows from the definition of $\TC$ and $\DTC$, so we aim to prove the inequality. By Theorem \ref{main_result} and Lemma \ref{tc_dtc_facts}, we have \begin{align*}
        \mathbb E_{S_1, \ldots, S_k}[\KL \mu {\nu^{S_1, \ldots, S_k}}] &= \sum_{i=1}^k \left(\sum_{j=1}^{s_i} \left(Z_{N_{i-1} + j} - Z_{N_{i-1} + 1}\right)\right) \\ &\leq \sum_{i=1}^k (s_i-1)(Z_{N_i} - Z_{N_{i-1}}) \\ &\leq (s_{\mathrm{max}}-1)Z_n \\ &\leq \frac{2^{\lceil \log_2 s_{\mathrm{max}}\rceil} - 1}{n}(\TC + \DTC),
    \end{align*} where in the second line we have noted that the summand $Z_{N_{i-1} + j} - Z_{N_{i-1} + 1}$ is zero for $j = 1$ and at most $Z_{N_i} - Z_{N_{i-1}}$ otherwise. This completes the proof.
\end{proof}
\begin{remark}
We can restate Theorem \ref{gen_li_result} in terms of query complexity given a fixed $\epsilon$. In particular, the number of queries is $k \geq \frac{n}{s_{\mathrm{max}}}$. Thus, we find that given any $\epsilon$, there is a mask schedule attaining expected sampling error at most $\epsilon$ in $\calO\left(\left\lceil\frac{\TC + \DTC}{\epsilon}\right\rceil\right)$ queries.
\end{remark}

As shown in \cite{li2025convergence}, this upper bound is optimal in some cases. However, as we remarked in the introduction, it is generally worse than Theorem \ref{thm:tcdtc}.

\subsection{Recovering the bound of Austin~\cite{austin2020multi}}
\label{sec:austin}

In \cite{austin2020multi}, the author proves that distributions with low $\TC$ can be decomposed into a fixed subset of coordinates $S$ and a remaining subset of coordinates $[n] \backslash S$ which have low conditional $\TC$. For completeness, we show the result here.

\begin{lemma}[Lemma 8.3 of \cite{austin2020multi}]\label{lem:austin_orig}
    Let $\mu$ be the data distribution. Then there is a subset size $s \leq \frac{\DTC}{\delta^2}$ for which in expectation over all $|S|= s, S \subseteq [n]$, we have  $$\TC(X_1, \ldots, X_n \mid X_S) + \DTC(X_1, \ldots, X_n \mid X_S) \leq \delta^2 (n - |S|),$$ where $\TC(Y \mid X) = \mathbb E_{x \sim p(X)} \TC(Y \mid X = x)$ is the conditional $\TC$ and similarly for $\DTC(Y \mid X)$. 
\end{lemma}

This lemma yields a natural method of sampling $\mu$: first perfectly sample an arbitrary subset of size $s$, and then sample the remaining coordinates in one-shot.

\begin{corollary}\label{austin_for_sampling}
    Suppose $\DTC \leq \delta^2 n$. Let $\mu$ be the data distribution and $\nu$ be the output of the random unmasking algorithm $\calA(k, \{s_i\}_1^k)$. Then there exists a schedule $(k, \{s_i\}_1^k)$ for which the query complexity satisfies $k \leq \frac{\DTC}{\delta^2} + 1$ and the error is bounded by $$\KL \mu {\nu} \leq \delta^2 (n - k+1).$$
\end{corollary}

Note that here $k = s+1$, since there is the final one-shot step. We provide a short proof of this result via Theorem \ref{main_result}.

\begin{proof}
Consider the schedule given by $k = \left\lfloor \frac \DTC{\delta^2}\right\rfloor + 1$, $s_i = 1$ for $i \leq k-1$, and $s_k = n - k+1$. By Theorem \ref{main_result}, we have that \allowdisplaybreaks \begin{align*}
    \KL \mu \nu &= (n-k+1)(Z_n - Z_k) \\ &\leq \frac{n-k+1}{k} \sum_{j=1}^k (Z_n - Z_j) \\ &\leq \frac{n-k+1}{k} \DTC \\ &\leq \delta^2(n - k + 1),
\end{align*} where we have used Lemma \ref{tc_dtc_facts} repeatedly. 
\end{proof}

Note that the bound in Corollary \ref{austin_for_sampling} is not particularly strong; in particular, the sampling procedure is essentially two-step and does not provide significant flexibility to the choice of mask schedule. This can be improved by replacing the one-shot sample of $S_k$ with an $\ell$-step, constant mask size sampler. Under this regime, combining the above result and Theorem \ref{gen_li_result} below and then optimizing over $k$ and $\ell$, we can recover Theorem \ref{thm:austinquerycomplexity} given in the introduction. We provide a detailed proof below.

\begin{proof}[Proof of Theorem \ref{thm:austinquerycomplexity}]
    Consider the schedule $s_i = 1$ for $i \leq k-1$ and $s_i = \lfloor \frac{n-k+1}{\ell}\rfloor $ for $k \leq i \leq k+\ell+1$.\footnote{This is approximate, as we do not necessarily have $\sum s_i = n$. Formally, if $\sum_i s_i > n$, omit as many terms $s_i$ from the end as necessary and cap the final value of $s_i$} Applying Theorem \ref{gen_li_result} to the conditional distribution $X_1, \ldots, X_n \mid X_S$ and then Lemma \ref{lem:austin_orig}, we find that the total error is bounded by $$\mathbb E\left[\KL \mu \nu\right] \leq \frac 1\ell (\TC (X_1, \ldots, X_n \mid X_S) + \DTC(X_1, \ldots, X_n \mid X_S)) \leq \frac{\delta^2n}{\ell}.  $$ The total query complexity is $k + \ell$. Take $\ell = \left\lceil \frac{\delta^2n}{\epsilon}\right\rceil$ and $k \leq \frac{\DTC}{\delta^2}$, observe that we can then set $\delta^2 = \sqrt{\frac{\DTC\cdot \epsilon}{n}}$. We finally find that this schedule yields error $\mathbb E\left[\KL \mu \nu\right] \leq \epsilon$ and query complexity $k+\ell = \mathcal O\left(\sqrt{\frac{\DTC \cdot n}{\epsilon}}\right)$, as desired. This completes the proof. 
\end{proof}

\section{Decoupling estimation error and sampling error}
\label{app:decoupling}

This appendix follows the work of \cite{li2025convergence}, and is written in order that the present work be self-contained. For simplicity, let $\mu(x)$ denote the PDF of $\mu$. Let $T_i = \cup_{j < i} S_j$ for any subsets $S = \bigsqcup_{i = 1}^k S_i$. We consider learning the estimate $\widehat{\mathsf{CO}}$ which minimizes the following error: \begin{align*}\text{error}(\mu, \widehat{\mathsf{CO}})= \mathbb E_{S_1, \ldots, S_k; i \in [k]} \left[\frac n{|S_i|} \sum_{j \in S_i} \log \frac{\mu(X_j \mid X_{T_i} = x_{T_i})}{\widehat{\mathsf{CO}}(X_j \mid X_{T_i} = x_{T_i})}\right]\end{align*} where $\widehat{\mathsf{CO}}(X_j \mid X_{T_i} = x_{T_i})$ denotes the conditional marginal of $X_j$ outputted by the learned oracle, and the expectation is over \textit{all} schedules $S = \bigsqcup_{i=1}^k S_i$, and $i$ is drawn from the distribution $p(i) = \frac{|S_i|}{n}$. 

As a brief remark, note that since $\widehat{\mathsf{CO}}$ only appears in the denominator of the logarithmic term, we can estimate the learning error $$\text{learning error} = -\mathbb E_{S_1, \ldots, S_k; i \in [k]} \left[\frac n {|S_i|} \sum_{j \in S_i} \log \widehat{\mathsf{CO}} (X_j \mid X_{T_i} = x_{T_i})\right]$$ via training samples; moreover, by positivity of KL, the optimum is precisely $\widehat{\mathsf{CO}} = \mathsf{CO}$, hence with sufficiently many samples, we can expect to learn a good estimate $\widehat{\mathsf{CO}}$.  

Now, the following error-decoupling result justifies our assumption that the conditional oracle $\widehat{\mathsf{CO}}$ is perfect. 

\begin{lemma}[\cite{li2025convergence}]
    Let $\mu$ be the data distribution and $(S_1, \ldots, S_k)$ be an unmasking schedule. Moreover, let $\widehat{\mathsf{CO}}$ be a learned conditional marginal oracle, which estimates $\mathsf{CO}$. Let $\nu^{S_1, \ldots, S_k}$ be the distribution sampled by $\calA_{\mathrm{fixed}}(k, \{S_1\}_1^k)$ using $\mathsf{CO}$, and $\hat\nu^{S_1, \ldots, S_k}$ is the distribution sampled by $\calA_{\mathrm{fixed}}(k, \{S_1\}_1^k$ using $\widehat{\mathsf{CO}}$. Then $$\KL\mu {\nu^{S_1, \ldots, S_k})} = \KL\mu {\hat\nu^{S_1, \ldots, S_k}} + \emph{error}(\mu, \widehat{\mathsf{CO}})$$
\end{lemma}
\begin{proof}
     Observe that \begin{align*}
        \KL\mu {\nu^{S_1, \ldots, S_k}} - \KL\mu {\hat\nu^{S_1, \ldots, S_k}}&= \int_{\Sigma^n} \mu(x) \log \left(\frac{ \nu^{S_1, \ldots, S_k}(x)}{\hat\nu^{S_1, \ldots, S_k}(x)}\right)\mathrm dx \\ &= \sum_{i = 1}^k\int_{\Sigma^n} \mu(x) \log \left(\frac{\nu^{S_1, \ldots, S_k}(x_{S_i} \mid X_{T_i} = x_{T_i})}{\hat \nu^{S_1, \ldots, S_k}(x_{S_i} \mid X_{T_i} = x_{T_i})}\right)\mathrm dx \\ &= \sum_{i = 1}^k \sum_{j \in S_i} \log \frac{\mu(X_j \mid X_{T_i} = x_{T_i})}{\widehat{\mathsf{CO}}(X_j \mid X_{T_i} = x_{T_i})}\\ &= \mathbb E_{S_1, \ldots, S_k; i \in [k]} \left[\frac n{|S_i|} \sum_{j \in S_i} \log \frac{\mu(X_j \mid X_{T_i} = x_{T_i})}{\widehat{\mathsf{CO}}(X_j \mid X_{T_i} = x_{T_i})}\right]\\ &= \text{error}(\mu, \widehat{\mathsf{CO}}),
    \end{align*} where the factor of $\frac n{|S_i|}$ is due to the distribution of $i$ in the expectation formula for $\mathrm{error}(\mu, \widehat{\mathsf{CO}})$. 
\end{proof}

\section{Finishing the proof of Theorem~\ref{thm:elevate}}
\label{app:elevate}

\begin{proof}
    It remains to verify that the family $\calF$ constructed in Section~\ref{sec:elevate} satisfies the second part of Theorem~\ref{thm:elevate}. For this, we closely follow the proof of Theorem~\ref{thm:warmup}. By Proposition~\ref{prop:independent}, if one queries the conditional marginal oracle for $\mu^*[\calU_\subspace]$ on a partial assignment of size $< k$, then the response will be identical to the one for $\mu^*[\calU]$, and likewise if the assignment is of size $> k$ and its projection to the $\mathbb{F}^n_q$ component is incompatible with any element of $\subspace$. We will use the same \emph{hit} and \emph{miss} terminology from before.

    Let $\calD^{\mu^*}_k$ denote the distribution over $\mu^*[\calU_\subspace]$ where $\subspace$ is a random $k$-dimensional RS code. Let $\calD$ denote the mixture distribution over $\calF$ given by
    \begin{equation}
        \frac{1}{2}\delta_{\mu^*[\calU]} + \frac{1}{2n-2} \sum^{n-1}_{k=1} \calD^\mu_k\,.
    \end{equation}
    As before, let $\leaves^*\coloneqq \leaves^{\calT(\mu^*[\calU])}(\mu^*[\calU])$. We must have
    \begin{equation}
        \TV\Bigl(\mu^*[\calU], \sum_{\ell\in\leaves^*} \Pr[\calA]{\ell\mid \mu^*[\calU]}\cdot \nu_\ell\Bigr) \le 1/8\,,
    \end{equation}
    or else $\E[\mu^*\sim\calD]{\costTV_\calT(\calA;\mu)} > 1/16$.

    For any leaf node $\ell$, let $v_1\to w_1\to v_2\to\cdots\to w_{T-1} \to v_T$ denote the sequence of decision and leaf nodes along the root-to-leaf path to $\ell$, and suppose the edges $(v_i, w_i)$ are labeled with partial assignments $X_{S^{(i)}} = x^{(i)}$. If $\ell\in\leaves^*$, then the edges $(w_i, v_{i+1})$ are labeled with $\nu\otimes \mathrm{Unif}(\mathbb{F}_q)^{\otimes (n - |S^{(i)}|)}$ for some distribution $\nu$ over $\Sigma^n$.

    Let $k_1\le \cdots \le k_T$ denote the numbers $|S^{(1)}|,\ldots,|S^{(T)}|$ in sorted order. For any $\subspace$ of dimension $k > k_T$ we have $\Pr[\calA]{\ell \mid \mu^*[\calU_\subspace]} = \Pr[\calA]{\ell \mid \mu^*[\calU]}$. For $k_j < k < k_{j+1}$, by Lemma~\ref{lem:RS_balanced}, Eq.~\eqref{eq:avoid} holds as before, and if $\ell$ avoids $\subspace$, the oracles output under every query along the path is of the form $\nu\otimes \mathrm{Unif}(\mathbb{F}_q)^{\otimes (n - |S|)}$ for some distribution $\nu$. In this case, again we have $\Pr[\calA]{\ell\mid \mu^*[\calU_\subspace]} = \Pr[\calA]{\ell \mid\mu^*[\calU]}$. The same reasoning applies to $k < k_1$.

    Let us write
    \begin{equation}
        \mathbb{E}_{\subspace\sim\calD_k}\TV\Bigl(\mu^*[\calU_\subspace] , \sum_{\ell\in\leaves(\mu^*[\calU_\subspace])} \Pr[\calA]{\ell \mid \mu^*[\calU_\subspace]}\cdot \nu_\ell\Bigr) \ge 1/2 - \mathbb{E}_{\subspace\sim\calD_k} \TV\Bigl(\mu^*[\calU], \sum_{\ell \in \leaves(\mu^*[\calU_\subspace])} \Pr[\calA]{\ell\mid\mu^*[\calU_\subspace]}\cdot \nu_\ell\Bigr)
    \end{equation}
    where we used that $\TV(\mu^*[\calU], \mu^*[\calU_\subspace]) \ge \TV(\calU, \calU_\subspace) \ge 1/2$ for any proper subspace $\subspace$. We can rewrite the mixture on the right-hand side as
    \begin{multline}
        \sum_{\ell \in \leaves^*:  \text{avoids} \ \subspace} \Pr[\calA]{\ell\mid \mu^*[\calU_\subspace]}\cdot \nu_\ell + \sum_{\ell\in\leaves^*: \text{hits} \ \subspace}\Pr[\calA]{\ell\mid \mu^*[\calU_\subspace]}\cdot \nu_\ell + \sum_{\ell \in \leaves(\mu^*[\calU_\subspace])\backslash\leaves^*}\Pr[\calA]{\ell\mid\mu^*[\calU_\subspace]}\cdot \nu_\ell \\
        = \sum_{\ell\in\leaves^*} \Pr[\calA]{\ell\mid \mu^*[\calU]}\cdot \nu_\ell - \sum_{\ell\in\leaves^*: \text{hits} \ \subspace} \Pr[\calA]{\ell\mid\mu^*[\calU]}\cdot \nu_\ell + \sum_{\ell \in \leaves(\mu^*[\calU_\subspace])\backslash\leaves^*}\Pr[\calA]{\ell\mid\mu^*[\calU_\subspace]}\cdot \nu_\ell\,, \label{eq:mixture}
    \end{multline}
    where we used that for $\ell\in\leaves^*$ that avoid $\subspace$, $\Pr[\calA]{\ell\mid\mu^*[\calU]} = \Pr[\calA]{\ell\mid\mu^*[\calU_\subspace]}$, and for $\ell \in \leaves^*$ that hit $\subspace$, it must be that $\Pr[\calA]{\ell\mid \mu^*[\calU_\subspace]} = 0$ as the sampler under $\mu^*[\calU_\subspace]$ must deviate from the path that leads to $\ell$. As $\sum_{\ell\in\leaves(\mu^*[\calU_\subspace])\backslash\leaves^*}\Pr[\calA]{\ell\mid\mu^*[\calU_\subspace]} = \sum_{\ell\in\leaves^*:\text{hits} \subspace}\Pr[\calA]{\ell\mid\mu^*[\calU]}$, the TV between $\mu^*[\calU]$ and the mixture in Eq.~\eqref{eq:mixture} is thus upper bounded by $1/8 + \sum_{\ell\in\leaves^*: \text{hits} \ \subspace} \Pr[\calA]{\ell\mid \mu^*[\calU]}$, and thus
    \begin{equation}
        \mathbb{E}_{\subspace\sim\calD_k}\TV\Bigl(\mu^*[\calU_\subspace] , \sum_{\ell\in\leaves(\mu^*[\calU_\subspace])} \Pr[\calA]{\ell \mid \mu^*[\calU_\subspace]}\cdot \nu_\ell\Bigr) \ge \frac{3}{8} - \sum_{\ell\in\leaves^*: \text{hits} \ \subspace} \Pr[\calA]{\ell\mid \mu^*[\calU]}\,.
    \end{equation}

    We say that $\subspace$ is \emph{$\eta$-good} if it satisfies $\sum_{\ell\in \leaves^*: \text{hits} \subspace} \Pr[\calA]{\ell\mid \mu^*[\calU]} \le \eta$ for some $\eta > 0$. Observe that 
    \begin{align}
        \MoveEqLeft\frac{1}{n-1}\sum^{n-1}_{k=1} \Bigl\{\sum_{\ell\in\leaves^*}\Pr[\calA]{\ell\mid \mu^*[\calU]} \cdot \Pr[\subspace\sim\calD_k]{\ell \ \text{hits} \ \subspace}\Bigr\} \\
        &= \sum_{\ell\in\leaves^*} \Pr[\calA]{\ell\mid \mu^*[\calU]} \cdot \frac{1}{n-1}\sum^{n-1}_{k=1} \Pr[\subspace\sim\calD_k]{\ell \ \text{hits} \ \subspace} \\
        &\le \sum_{\ell\in\leaves^*} \Pr[\calA]{\ell \mid \mu^*[\calU]} \cdot \frac{\calT(\mu^*[\calU]) + (n - 1 - \calT(\mu^*[\calU])) \calT(\mu^*[\calU])/q}{n-1} \\
        &= \frac{\calT(\mu^*[\calU]) + (n - 1 - \calT(\mu^*[\calU])) \calT(\mu^*[\calU])/q}{n-1} \le \frac{2\calT(\mu^*[\calU])}{n - 1}
    \end{align}
    where in the second step we used that for any leaf $\ell$ at distance $2T$ from the root, there are at most $T$ dimensions $0 < k < n$ that are equal to the size of some partial assignment along the root-to-left path to $\ell$, and for all other dimensions $k$, $\Pr[\subspace\sim \calD_k]{\ell \ \text{hits} \ \subspace} \le T/q$ by Eq.~\eqref{eq:avoid}. By Markov's inequality, we conclude that for $\eta \coloneqq \frac{4\calT(\mu^*[\calU])}{n - 1} \ll 1$ (here we used the hypothesis from Eq.~\eqref{eq:assumesmall}),
    \begin{equation}
        \Pr[0<k<n, \subspace\sim\calD_k]{\subspace \ \text{is} \ \eta\text{-good}} \ge 1/2 \,.
    \end{equation}
    We conclude that
    \begin{align}
        \E[\mu\sim\calD]{\costTV_\calT(\calA;\mu)} &\ge \frac{1}{2}\cdot \Pr[0<k<n, \subspace\sim\calD_k]{\subspace \ \text{is} \ \eta\text{-good}}\cdot \Bigl(\frac{3}{8} - \eta\Bigr) \ge \frac{1}{16}\,.\qedhere
    \end{align}
\end{proof}

\end{document}